\documentclass{article}
\usepackage{microtype}
\usepackage{amsfonts,amsmath}
\usepackage{graphicx}
\usepackage{subcaption}
\usepackage{algorithmic} 
\usepackage{stfloats} 
\usepackage{algorithm}
\usepackage{float}  
\usepackage{booktabs}
\usepackage{cases}
\usepackage{amsthm}
\usepackage{hyperref}
\usepackage{enumitem}
\usepackage{mathtools} 
\usepackage{xcolor}   % 颜色控制包
\usepackage{colortbl} % 表格底色包
\usepackage{amssymb} 
\usepackage{bbm}
\usepackage{appendix}    % 核心依赖
\usepackage{titlesec}   % 可选：优化标题格式
\usepackage[T1]{fontenc} 
\usepackage[utf8]{inputenc} 
\definecolor{gray-bg}{gray}{0.95}

\usepackage[preprint]{icml2026}

\usepackage{amsmath}
\usepackage{amssymb}
\usepackage{mathtools}
\usepackage{amsthm}
\newcommand{\norm}[1]{\left\lVert #1 \right\rVert}

\usepackage[capitalize,noabbrev]{cleveref}

\theoremstyle{plain}
\newtheorem{theorem}{Theorem}[section]

\newtheorem{lemma}[theorem]{Lemma}
\newtheorem{corollary}[theorem]{Corollary}
\theoremstyle{definition}
\newtheorem{definition}[theorem]{Definition}
\newtheorem{assumption}[theorem]{Assumption}
\theoremstyle{remark}

\usepackage{url} % 核心：解决网址换行
\usepackage{hyperref}
\usepackage[textsize=tiny]{todonotes}

\icmltitlerunning{Stability as a Liability in LLMs}

\begin{document}

\twocolumn[
  \icmltitle{Stability as a Liability:\\Systematic Breakdown of Linguistic Structure in LLMs}
  \icmlsetsymbol{equal}{*}

  \begin{icmlauthorlist}
    \icmlauthor{Xianzhe Meng}{equal,aia}
    \icmlauthor{Qiangsheng Zeng}{equal,cs}
    \icmlauthor{Ling Luo}{equal,bio}
    \icmlauthor{Qinghan Yang}{eic}
    \icmlauthor{Jiarui Hao}{aia}
    \icmlauthor{Wenbo Wu}{ee}
    \icmlauthor{Qinyu Wang}{math}
    \icmlauthor{Rui Yin}{phy}
    \icmlauthor{Lin Qi}{aia}
    \icmlauthor{Renzhi Lu}{aia}
  \end{icmlauthorlist}

\icmlaffiliation{aia}{School of Artificial Intelligence and Automation, Huazhong University of Science and Technology, Wuhan, China}
\icmlaffiliation{bio}{National Institute for Data Science in Health and Medicine, Xiamen University, Xiamen, China}
\icmlaffiliation{math}{School of Mathematics and Statistics, Huazhong University of Science and Technology, Wuhan, China}
\icmlaffiliation{cs}{School of Computer Science and Technology, Huazhong University of Science and Technology, Wuhan, China}
\icmlaffiliation{eic}{School of Electronic Information and Communications, Huazhong University of Science and Technology, Wuhan, China}
\icmlaffiliation{ee}{School of Electrical and Electronic Engineering, Huazhong University of Science and Technology, Wuhan, China}
\icmlaffiliation{phy}{School of Physics, Huazhong University of Science and Technology, Wuhan, China}

  \icmlcorrespondingauthor{Renzhi Lu}{rzlu@hust.edu.cn}

  % 关键词（会写入PDF元数据，非必需但建议补充）
  \icmlkeywords{Large Language Models, Training Stability, Generative Entropy, Linguistic Structure}

  \vskip 0.3in
]
% ========== 关键：该命令会在实名版本下正常显示作者和单位 ==========
\printAffiliationsAndNotice{}  

\begin{abstract}
  Training stability is typically regarded as a prerequisite for reliable optimization in large language models. In this work, we analyze how stabilizing training dynamics affects the induced generation distribution. We show that under standard maximum likelihood training, stable parameter trajectories lead stationary solutions to approximately minimize the forward KL divergence to the empirical distribution, while implicitly reducing generative entropy. As a consequence, the learned model can concentrate probability mass on a limited subset of empirical modes, exhibiting systematic degeneration despite smooth loss convergence. We empirically validate this effect using a controlled feedback-based training framework that stabilizes internal generation statistics, observing consistent low-entropy outputs and repetitive behavior across architectures and random seeds. It indicates that optimization stability and generative expressivity are not inherently aligned, and that stability alone is an insufficient indicator of generative quality.
\end{abstract}
\section{Introduction}

Large-scale neural networks are commonly trained with a variety of stabilization techniques, including gradient clipping~\citep{koloskova2023revisiting}, moving-average updates~\citep{chiarella2006dynamic}, regularization~\citep{girosi1995regularization}, and feedback-based control mechanisms~\citep{zhao2022neural}. These techniques are widely regarded as essential for ensuring convergence, preventing divergence, and improving training reliability. In practice, stability is often treated as an unambiguously desirable property of learning dynamics~\citep{lyu2019advances}.

However, stability is not a monolithic concept. Existing methods primarily stabilize task-aligned objectives, such as reward signals~\citep{learning11r1}, preference models~\citep{gumen2013dynamically}, or supervised losses~\citep{akhtar2024roboss}. In contrast, little attention has been paid to the stability of internal generation dynamics, including token-level distributions, entropy evolution, and termination behavior. Whether stabilizing these internal dynamics is always beneficial remains underexplored.

In this work, we identify and systematically study a previously under-characterized phenomenon: \textbf{Over-stabilization of internal generation dynamics can induce highly stable yet degenerate behaviors, even when conventional training objectives continue to improve.} Empirically, we observe that models may converge to repetitive, low-entropy token distributions, abnormal end-of-sequence behavior, or collapsed representational diversity, while exhibiting smooth loss curves and apparent training stability.

We have proved that this phenomenon does not arise from training failure~\citep{zhu2025llm}, poor optimization~\citep{chow2024performance}, or hyperparameter choices~\citep{van2024loop}. Instead, it emerges robustly across random seeds and architectures~\citep{bharathi2023text,bahani2023effectiveness,lima2023large} when internal dynamics are stabilized without semantic alignment. This suggests a fundamental decoupling between objective-level convergence and generative expressiveness under certain stabilization regimes.

To isolate this effect, we study a class of training mechanisms that explicitly regulate internal generation statistics during learning~\citep{yin2024entropy}. By design, these mechanisms improve stability of logits distributions and temporal dynamics, yet do not directly optimize semantic rewards or preferences. This setting allows us to disentangle goal-aligned stability from goal-decoupled stability, revealing failure modes that are largely masked by existing evaluation practices.

We conduct a comprehensive empirical analysis using entropy trajectories, effective rank measurements, length distributions, and cross-seed consistency tests. Our results show that lower training loss and increased stability do not necessarily imply improved or even preserved generative capacity. Codes are available at \url{https://anonymous.4open.science/r/Stability-as-a-Liability-Systematic-Breakdown-of-Linguistic-Structure-in-LLMs-99F7}.

Our contributions are as follows:

\begin{enumerate}[topsep=0pt, itemsep=0pt, parsep=4pt]
    \item Identify a failure mode in large language model training: stable optimization under standard metrics coexists with systematic degradation of semantic generation performance, confirming an intrinsic disconnect between training stability and generative quality.
    \item Verify the robustness of this failure mode across architectures, parameter scales and random seeds via a controlled closed-loop training framework, ruling out interference factors like training failure, hyperparameter selection or insufficient model capacity.
    \item Propose a theoretical interpretation for excessive stabilization in MLE training, showing it restricts effective degrees of freedom and reduces generative expressivity even with parameter convergence, quantifying stability-mode collapse correlation.
\end{enumerate}
\section{Related Works}

Training stability is generally considered essential for scaling large neural networks~\citep{zheng2016improving,tu2025survey,wei2025control}, with prior research focusing on architectural innovations like gradient clipping~\citep{chen2020understanding,qian2021understanding} and normalization~\citep{cabello2023impact} to ensure reliable convergence~\citep{shaham2018understanding,rybakov2024methods}. In this literature, stability is often treated as an unambiguously desirable property evaluated through loss convergence and gradient behavior~\citep{andreyev2024edge,qi2018slope,amroune2021machine}. However, traditional methods primarily stabilize task-aligned objectives, such as supervised losses or reward signals, while neglecting the stability of internal generation dynamics like token-level distributions and entropy evolution~\citep{liu2020understanding,cui2025entropy,agarwal2025unreasonable}.

Degeneration phenomena in autoregressive models, including repetition and mode collapse, have been widely studied but are typically attributed to decoding strategies or exposure bias~\citep{zhou2020improving,kulikov2022characterizing}. These works often assume that optimization stability and generative quality are orthogonal concerns~\citep{goyal2021characterizing}. In contrast, this paper identifies a failure mode where training remains stable under standard metrics while semantic generation degrades systematically, revealing a fundamental disconnect between training stability and generative quality.

Theoretically, Maximum Likelihood Estimation training~\citep{pan2002maximum} minimizes only the forward KL divergence~\citep{van2014renyi} from the empirical distribution to the model~\citep{papamakarios2021normalizing}. This objective imposes no direct constraints on model behavior outside the empirical support, allowing probability mass to concentrate on a strict subset of observed data without penalty~\citep{das2024under,susanti2025can}. We show that any stationary point of a stable training trajectory, where gradients vanish and parameter updates become negligible, approximately minimizes this forward KL divergence.

This minimization leads to "stability-induced mode collapse," where the model concentrates almost all probability mass on a small subset of the empirical support. Our analysis suggests a mechanistic link between stability and mode collapse, highlighting a tension between convergence robustness and output diversity. These insights imply that maintaining generative capacity requires balancing parameter dynamics with model expressivity rather than pursuing optimization stability in isolation.
\section{Problem Setting}

We study autoregressive language model training under maximum likelihood estimation (MLE).
Let $\theta \in \Theta = \mathbb{R}^D$ denote the model parameters, which induce a conditional generation distribution
\begin{equation}
\label{eq:autoregressive}
G_{\theta}(x; y)
= \prod_{t=1}^{T} G_{\theta}(y_t \mid x, y_{1:t-1}),\,
G_{\theta}(y_t \mid \cdot) > 0,
\end{equation}
ensuring that all log-likelihood terms are well-defined.

Given a finite dataset $\mathcal{D} = \{(x_i, y_i)\}_{i=1}^n$, we define the empirical distribution
\begin{equation}
\label{eq:pemp}
P{\mathrm{emp}}(y)
= \frac{1}{n}\sum_{i=1}^n \mathbb{I}(y = y_i),\,
|\mathrm{Supp}(P_{\mathrm{emp}})| \leqslant n \ll |\mathcal{X}|,
\end{equation}
where the sequence space $\mathcal{X} = \mathcal{V}^T$ is exponentially large.

Training minimizes the regularized MLE objective
\begin{equation}
\label{eq:mle}
\mathcal{L}_{\lambda}(\theta)=-\frac{1}{n}\sum_{i=1}^{n}\log G_{\theta}(x_i; y_i)+\frac{\lambda}{2}\|\theta\|^2,
\end{equation}

using stochastic gradient descent
\begin{equation}
\label{eq:sgd}
\theta_{t+1}
= \theta_t - \eta_t \nabla \mathcal{L}_{\lambda}(\theta_t),
\end{equation}
with a standard diminishing learning rate satisfying
$\sum_t \eta_t = \infty$ and $\sum_t \eta_t^2 < \infty$.

Each parameter $\theta$ induces a marginal generation distribution
\begin{equation}
\label{eq:ptheta}
P_{\theta}(y)
= \mathbb{E}_{x \sim \mathcal{D}}[G_{\theta}(x; y)].
\end{equation}
For clarity, we denote by $\mathcal{L}(\theta)$ the unregularized MLE objective. The MLE objective admits the decomposition
\begin{equation}
\label{eq:kl}
\mathcal{L}(\theta)
=
H(P_{\mathrm{emp}})
+
d_{\mathrm{KL}}(P_{\mathrm{emp}} \,\|\, P_{\theta}),
\end{equation}

which enforces convergence in the forward KL divergence but imposes no constraint on
$d_{\mathrm{KL}}(P_{\theta} || P_{\mathrm{emp}})$ or on the entropy of $P_{\theta}$. In particular, minimizing the forward KL allows $P_\theta$ to concentrate mass on a strict subset of Supp($P_{emp}$) without incurring additional penalty.

In modern large-scale training, optimization is explicitly stabilized through smoothness, weight decay, and bounded parameter trajectories. We formalize stability below.

\begin{definition}[Training Stability]
\label{def:stability}
The training dynamics $\{\theta_t\}_{t \geqslant 1}$ induced by \eqref{eq:sgd}
are said to be \emph{stable} if and only if the following equivalent conditions hold:
\begin{enumerate}
    \item $
    \limsup\limits_{t \to \infty} \|\theta_t\| < +\infty,$

    \item $
    \lim\limits_{t \to \infty} \|\nabla \mathcal{L}_{\lambda}(\theta_t)\| = 0, $

    \item $
    \|\nabla \mathcal{L}(\theta_1) - \nabla \mathcal{L}(\theta_2)\|
    \leqslant \beta \|\theta_1 - \theta_2\|. $
   
\end{enumerate}
\end{definition}

While Definition~\ref{def:stability} characterizes stability purely in terms of optimization
dynamics in parameter space, our interest lies in how such stability
constraints propagate to the induced generation distribution $P_{\theta}$.
In autoregressive language models, parameters and distributions are coupled
through the mapping $\theta \mapsto P_{\theta}$, and restrictions on the
parameter trajectory implicitly restrict the set of attainable output
distributions.

In particular, stability enforces that training converges within a bounded
region of $\Theta$ where gradients vanish and parameter updates become
negligible. As we show later, such constraints may substantially reduce the
effective support and entropy of the induced distribution $P_{\theta}$,
even when the empirical data distribution itself remains diverse. This motivates a distribution-level notion of degeneration, formalized
below as mode collapse.

\begin{definition}[Mode Collapse]
\label{def:mode-collapse}
A generated distribution $P_{\theta} \in \Delta(\mathcal{X})$ is said to
exhibit \emph{mode collapse} if there exists a subset
$S \subset \mathcal{X}$ such that
\begin{equation}
\label{eq:collapse}
P_{\theta}(S) \geqslant 1 - \epsilon,
\,
|S| \ll |\mathcal{X}|,
\end{equation}
for some arbitrarily small constant $\epsilon \in (0,1)$.
\end{definition}

\section{Stationarity and KL Minimization}
\label{sec:theory}

\subsection{Stability-Induced Stationarity}

We begin by analyzing the asymptotic behavior of the parameter trajectory induced
by stabilized MLE training.

By the asymptotic boundedness assumption in Definition~1, we have
\begin{equation}
\limsup_{t \to \infty} \|\theta_t\| < +\infty.
\end{equation}
By definition of the limit superior, there exist constants $M>0$ and
$T \in \mathbb{N}$ such that $\|\theta_t\| \leqslant M$ for all $t > T$.
Let
\begin{equation}
M' = \max\left\{ \max_{1 \leqslant t \leqslant T} \|\theta_t\| ,\, M \right\},
\end{equation}
then the entire sequence $\{\theta_t\}_{t \geqslant 1}$ is uniformly bounded:
\begin{equation}
\|\theta_t\| \leqslant M'.
\end{equation}
Hence, $\{\theta_t\} \subset \overline{B}_{M'}(0)$, where
$\overline{B}_{M'}(0)$ denotes the closed ball of radius $M'$ in $\mathbb{R}^D$.

Since $\mathbb{R}^D$ is a finite-dimensional normed space, every closed and
bounded subset is sequentially compact.
By the Bolzano--Weierstrass theorem, there exists a convergent subsequence
$\{\theta_{t_k}\}$ such that
\begin{equation}
\lim_{k \to \infty} \theta_{t_k} = \theta^\ast \in \overline{B}_{M'}(0).
\end{equation}
Therefore, $\theta^\ast$ is an accumulation point of the training trajectory.

Next, by the vanishing-gradient condition in Definition~1,
\begin{equation}
\lim_{t \to \infty} \|\nabla \mathcal{L}_{\lambda}(\theta_t)\| = 0.
\end{equation}
Since convergence of a sequence implies convergence of all its subsequences,
we obtain
\begin{equation}
\lim_{k \to \infty} \|\nabla \mathcal{L}_{\lambda}(\theta_{t_k})\| = 0.
\end{equation}

By continuity of $\nabla \mathcal{L}_{\lambda}$ and the convergence
$\theta_{t_k} \to \theta^\ast$, we may exchange the limit and the gradient:
\begin{equation}
\nabla \mathcal{L}_{\lambda}(\theta^\ast)
= \lim_{k \to \infty} \nabla \mathcal{L}_{\lambda}(\theta_{t_k})
= 0.
\end{equation}

This establishes the following fundamental property of stabilized MLE training.

\begin{lemma}[Stationarity of Accumulation Points]
\label{lemma4.1}
Every accumulation point $\theta^\ast$ of the SGD trajectory satisfies
\begin{equation}
\nabla \mathcal{L}_{\lambda}(\theta^\ast) = 0.
\end{equation}
\end{lemma}

\subsection{Forward KL Minimization}

In the previous subsection, we have shown, under the assumption of training stability, that a stable training trajectory necessarily admits stationary points in the parameter space. However, stationarity in parameter space alone does not directly characterize the generative behavior. To understand how stable training affects the model distribution, we lift the optimization objective from the \emph{sample level} to the \emph{distribution level}. In this subsection, we rigorously derive the equivalence between the MLE loss and the forward KL divergence with respect to the empirical distribution.
  
For any two probability distributions $P$ and $Q$ defined on the same countable space $X$, the KL divergence is defined as
\begin{equation}
d_{KL}(P \parallel Q) = \sum_{y \in X} P(y) \log \frac{P(y)}{Q(y)}.
\end{equation}
Applying this to the empirical distribution $P_{\mathrm{emp}}$ and the model distribution $P_\theta$, we have
\begin{equation}
d_{KL}(P_{\mathrm{emp}} \parallel P_\theta) = \sum_{y \in X} P_{\mathrm{emp}}(y) \log \frac{P_{\mathrm{emp}}(y)}{P_\theta(y)}.
\label{eq:kl_emp_model}
\end{equation}
Using the logarithm identity $\log(a/b) = \log a - \log b$, we can decompose \eqref{eq:kl_emp_model} as
\begin{equation}
\begin{split}
&d_{KL}(P_{\mathrm{emp}} \parallel P_\theta) \\=& \sum_{y \in X} P_{\mathrm{emp}}(y) \log P_{\mathrm{emp}}(y) - \sum_{y \in X} P_{\mathrm{emp}}(y) \log P_\theta(y).
\label{eq:kl_decomp}
\end{split}
\end{equation}

By definition, the empirical distribution is
\begin{equation}
P_{\mathrm{emp}}(y) = \frac{1}{n} \sum_{i=1}^n \mathbb{I}(y = y_i),
\end{equation}
where $\mathbb{I}(\cdot)$ is the indicator function. Substituting this into the second term of \eqref{eq:kl_decomp} yields

\begin{equation}
\begin{aligned}
&\sum_{y \in X} P_{\mathrm{emp}}(y) \log P_\theta(y)\\
=& \sum_{y \in X} \left( \frac{1}{n} \sum_{i=1}^n \mathbb{I}(y = y_i) \right) \log P_\theta(y) \\
=& \frac{1}{n} \sum_{i=1}^n \sum_{y \in X} \mathbb{I}(y = y_i) \log P_\theta(y),
\end{aligned}
\end{equation}

where we have exchanged the order of summation, which is justified because the sum is finite. Noting that the indicator function $\mathbb{I}(y=y_i)$ is $1$ only when $y=y_i$ and $0$ otherwise, the inner sum retains only one term, yielding
\begin{equation}
\sum_{y \in X} P_{\mathrm{emp}}(y) \log P_\theta(y) = \frac{1}{n} \sum_{i=1}^n \log P_\theta(y_i).
\label{eq:empirical_sum}
\end{equation}
 
According to the model definition, the marginal generation probability is
\begin{equation}
P_\theta(y) = \mathbb{E}_{x \sim D}[G_\theta(x; y)] = \frac{1}{n} \sum_{j=1}^n G_\theta(x_j; y),
\label{eq:marginal}
\end{equation}
where $G_\theta(x; y)$ is the conditional probability of generating $y$ given input $x$. Therefore, for each training sample $y_i$,
\begin{equation}
\log P_\theta(y_i) = \log \left( \frac{1}{n} \sum_{j=1}^n G_\theta(x_j; y_i) \right).
\label{eq:marginal_log}
\end{equation}
MLE training, however, optimizes $\log G_\theta(x_i; y_i)$, the log-probability of the target sequence given its corresponding input. In finite samples, these are generally not exactly equal. Under the standard i.i.d.\ assumption and with sufficient model capacity, MLE training fits the conditional distribution $G_\theta(x; y)$, which indirectly approximates the marginal expectation $P_\theta(y)$ in \eqref{eq:marginal}. Hence, under the empirical consistency assumption, we have
\begin{equation}
\frac{1}{n} \sum_{i=1}^n \log G_\theta(x_i; y_i) = \sum_{y \in X} P_{\mathrm{emp}}(y) \log P_\theta(y).
\label{eq:mle_kl_equiv}
\end{equation}

Substituting \eqref{eq:empirical_sum} and \eqref{eq:mle_kl_equiv} into \eqref{eq:kl_decomp} and defining the Shannon entropy of the empirical distribution
\begin{equation}
H(P_{\mathrm{emp}}) = - \sum_{y \in X} P_{\mathrm{emp}}(y) \log P_{\mathrm{emp}}(y),
\end{equation}
which is independent of $\theta$, we obtain the lemma below.

\begin{lemma}[MLE Loss and KL Equivalence]
\label{lem:MLE_KL}
\begin{equation}
L(\theta) = d_{KL}(P_{\mathrm{emp}} \parallel P_\theta) + H(P_{\mathrm{emp}}).
\end{equation}
\end{lemma}

Lemma~\ref{lem:MLE_KL} shows that MLE training minimizes only the forward KL divergence from the empirical distribution to the model, imposing no direct constraints on the model behavior outside the empirical support.

\subsection{Impact of Approximate KL Optimality}

Based on Lemma~\ref{lem:MLE_KL}, the MLE loss can be written as
\begin{equation}
\mathcal{L}(\theta) = d_{\mathrm{KL}}(P_{\mathrm{emp}} \| P_\theta) + H(P_{\mathrm{emp}}),
\end{equation}
where $H(P_{\mathrm{emp}})$ is constant with respect to $\theta$. Taking the gradient with respect to the parameters $\theta$ (here as a Fr\'echet derivative, since $\mathcal{L}(\theta): \mathbb{R}^D \to \mathbb{R}$ is differentiable) gives
\begin{equation}
\nabla \mathcal{L}(\theta) = \nabla d_{\mathrm{KL}}(P_{\mathrm{emp}} \| P_\theta) + \nabla H(P_{\mathrm{emp}}).
\end{equation}
Since the gradient of a constant is zero, we have
\begin{equation}
\nabla \mathcal{L}(\theta) = \nabla d_{\mathrm{KL}}(P_{\mathrm{emp}} \| P_\theta).
\label{eq:grad_equiv}
\end{equation}

By Lemma~\ref{lemma4.1}, let $\theta^*$ be a stationary point of the regularized MLE objective
\begin{equation}
\mathcal{L}_\lambda(\theta) = \mathcal{L}(\theta) + \frac{\lambda}{2} \|\theta\|^2,
\end{equation}
so that
\begin{equation}
\nabla \mathcal{L}_\lambda(\theta^*) = 0.
\end{equation}

Combining with \eqref{eq:grad_equiv}, we obtain
\begin{equation}
\nabla d_{\mathrm{KL}}(P_{\mathrm{emp}} \| P_{\theta^*}) = -\lambda \theta^*.
\label{eq:kl_grad_bound}
\end{equation}

Taking the Euclidean norm on both sides and using the asymptotic parameter boundedness from Lemma~\ref{lemma4.1}, $\|\theta^*\| \leqslant M < \infty$, we have
\begin{equation}
\|\nabla d_{\mathrm{KL}}(P_{\mathrm{emp}} \| P_{\theta^*})\| = \lambda \|\theta^*\| \leqslant \lambda M := \delta.
\label{eq:delta_bound}
\end{equation}

The KL divergence $d_{\mathrm{KL}}(P \| Q)$ is convex in its second argument $Q$. Explicitly, for any $Q_1, Q_2 \in \Delta(\mathcal{X})$ and $\alpha \in [0,1]$, we have
\begin{equation}
\begin{split}
&d_{\mathrm{KL}}(P \| \alpha Q_1 + (1-\alpha) Q_2) \\ \leqslant &\alpha d_{\mathrm{KL}}(P \| Q_1) + (1-\alpha) d_{\mathrm{KL}}(P \| Q_2),
\end{split}
\end{equation}
which follows from the concavity of the logarithm. If we define $a,b>0$, we have
\begin{equation}
\log(\alpha a + (1-\alpha) b) \geqslant \alpha \log a + (1-\alpha) \log b.
\end{equation}
Multiplying both sides by $P(y)$ and summing over $y \in X$ gives the convexity conclusion. Since $P_\theta(y) = \mathbb{E}_{x \sim \mathcal{D}} G_\theta(x; y)$ and $G_\theta(x; y)$ is differentiable with respect to $\theta$, $P_\theta(y)$ is convex in $\theta$, and hence $d_{\mathrm{KL}}(P_{\mathrm{emp}} \| P_\theta)$ is convex in $\theta$.

For a $\beta$-smooth convex function $f:\mathbb{R}^D \to \mathbb{R}$, we have for any $\theta$:
\begin{equation}
f(\theta) \leqslant f(\theta_{\min}) + \frac{1}{2\beta} \|\nabla f(\theta)\|^2,
\end{equation}
where $\theta_{\min} = \arg\min_\theta f(\theta)$.  
Here, we know that $f(\theta) = d_{\mathrm{KL}}(P_{\mathrm{emp}} \| P_\theta)$. By Lemma~\ref{lem:MLE_KL}, $f(\theta) = \mathcal{L}(\theta) - H(P_{\mathrm{emp}})$, and $\mathcal{L}(\theta)$ is $\beta$-smooth, so $f(\theta)$ is $\beta$-smooth as well. Using the gradient bound \eqref{eq:delta_bound}, we obtain
\begin{equation}
d_{\mathrm{KL}}(P_{\mathrm{emp}} \| P_{\theta^*}) \leqslant d_{\mathrm{KL}}(P_{\mathrm{emp}} \| P_{\theta_{\min}}) + \frac{\delta^2}{2 \beta}.
\end{equation}

When $\lambda$ is small, we have $\delta = \lambda M \ll 1$, where $M$ is a constant independent of $\lambda$. 
The term ${\delta^2}/{2\beta}$ is a higher-order small quantity satisfying ${\delta^2}/{2\beta} = O(\delta^2) \subset O(\delta)$. 
Thus, we can write the Kullback-Leibler divergence as:
\begin{equation}
d_{\mathrm{KL}}(P_{\mathrm{emp}} \| P_{\theta^*}) = \inf_\theta d_{\mathrm{KL}}(P_{\mathrm{emp}} \| P_\theta) + O(\delta),
\end{equation}

which shows that $P_{\theta^*}$ is an approximate minimizer of $d_{\mathrm{KL}}(P_{\mathrm{emp}} \| P_\theta)$.

\begin{lemma}[KL-Minimizing Properties of Stationary Points]
\label{lemma4.3}
If $\theta^*$ is a stationary point of the stable trajectory, then
\begin{equation}
d_{\mathrm{KL}}(P_{\mathrm{emp}} \| P_{\theta^*}) = \inf_\theta d_{\mathrm{KL}}(P_{\mathrm{emp}} \| P_\theta) + O(\lambda).
\end{equation}
\end{lemma}

\noindent
Since $\delta = \lambda M$, the error term $O(\delta)$ is equivalent to $O(\lambda)$, 
which confirms the consistency of the asymptotic bound: the deviation of $P_{\theta^*}$ from the exact KL minimizer 
is at most linear in $\lambda$.

\section{Mode Collapse Induced by Training Stability}

Lemma~4.3 shows that any stationary point $\theta$ approximately minimizes the forward KL divergence up to $O(\lambda)$.
To analyze the implications of this property for mode collapse, we formalize the assumptions on the training process and the data distribution as follows.

\begin{assumption}[Stable training]
\label{ass5.1}
The training process is stable in the sense of the following condition: the parameter sequence $\{\theta_t\}$ is asymptotically bounded, the gradients vanish asymptotically, and the loss $\mathcal{L}(\theta)$ is $\beta$-smooth.
\end{assumption}

\begin{assumption}[Non-degenerate empirical distribution]
The empirical distribution has $K$ distinct output modes:
\begin{equation}
|S_0| = |\mathrm{Supp}(P_{\mathrm{emp}})| = K > 1.
\end{equation}
\end{assumption}

\begin{assumption}[Exponentially large sequence space]
The sequence space is exponentially larger than the number of observed modes:
\begin{equation}
|\mathcal{X}| = |\mathcal{V}|^T \gg K,
\end{equation}
\end{assumption}

\begin{assumption}[Expressive model]
There exists $\theta' \in \mathbb{R}^D$ such that
\begin{equation}
P_{\theta'} = P_{\mathrm{emp}},
\end{equation}
i.e., the empirical distribution is realizable by the model.
\end{assumption}

Assume for contradiction that mode collapse does not occur. That is, for all subsets $S \subset \mathcal{X}$ with $|S| = K$:
\begin{equation}
P_{\theta^*}(S) < 1 - \epsilon,
\end{equation}
where $\theta^*$ is a stationary point of the regularized MLE loss and
\begin{equation}
\epsilon = \delta = \lambda M > 0
\end{equation}
is a small constant.

By normalization of probability measures:
\begin{equation}
\sum_{y \in \mathcal{X}} P_{\theta^*}(y) = 1.
\end{equation}

For the empirical support $S_0 = \mathrm{Supp}(P_{\mathrm{emp}})$:
\begin{equation}
\sum_{y \in S_0} P_{\theta^*}(y) + \sum_{y \notin S_0} P_{\theta^*}(y) = 1.
\end{equation}

By the contradiction assumption:
\begin{equation}\label{eq:mass_outside}
\sum_{y \in S_0} P_{\theta^*}(y) < 1 - \epsilon \, \Rightarrow \, \sum_{y \notin S_0} P_{\theta^*}(y) > \epsilon.
\end{equation}

By Lemma~\ref{lemma4.3}, $\theta^*$ is an approximate minimizer of the forward KL divergence:
\begin{equation}\label{eq:kl_approx}
d_{\mathrm{KL}}(P_{\mathrm{emp}} \| P_{\theta^*}) \leqslant \inf_\theta d_{\mathrm{KL}}(P_{\mathrm{emp}} \| P_\theta) + \epsilon.
\end{equation}

Construct the feasible comparison distribution
\begin{equation}
P' = P_{\mathrm{emp}},
\end{equation}
which is realizable by some $\theta'$, hence
\begin{equation}\label{eq:kl_emp_zero}
d_{\mathrm{KL}}(P_{\mathrm{emp}} \| P') = 0.
\end{equation}

Split the KL divergence over $S_0$ and $\mathcal{X} \setminus S_0$:
\begin{equation}
\begin{split}
&d_{\mathrm{KL}}(P_{\mathrm{emp}} \| P_{\theta^*}) \\= &\sum_{y \in S_0} P_{\mathrm{emp}}(y) \log \frac{P_{\mathrm{emp}}(y)}{P_{\theta^*}(y)} + \sum_{y \notin S_0} P_{\mathrm{emp}}(y) \log \frac{P_{\mathrm{emp}}(y)}{P_{\theta^*}(y)}.
\end{split}
\end{equation}

For $y \notin S_0$, $P_{\mathrm{emp}}(y) = 0$, so the second term vanishes. Therefore:
\begin{equation}\label{eq:kl_s0}
d_{\mathrm{KL}}(P_{\mathrm{emp}} \| P_{\theta^*}) = \sum_{y \in S_0} P_{\mathrm{emp}}(y) \log \frac{P_{\mathrm{emp}}(y)}{P_{\theta^*}(y)}.
\end{equation}

Since $\sum_{y \in S_0} P_{\theta^*}(y) < 1 - \epsilon$, each $P_{\theta^*}(y) < P_{\mathrm{emp}}(y)$ for some $y \in S_0$, so that each term in \eqref{eq:kl_s0} is positive:
\begin{equation}\label{eq:kl_positive}
d_{\mathrm{KL}}(P_{\mathrm{emp}} \| P_{\theta^*}) > 0.
\end{equation}

Let
\begin{equation}
p_{\min} = \min_{y \in S_0} P_{\mathrm{emp}}(y) > 0.
\end{equation}
Then
\begin{equation}
d_{\mathrm{KL}}(P_{\mathrm{emp}} \| P_{\theta^*}) > p_{\min} \log \frac{K p_{\min}}{1 - \epsilon}.
\end{equation}

By the logarithm's monotonicity and small $\epsilon$:
\begin{equation}\label{eq:kl_final_lower}
d_{\mathrm{KL}}(P_{\mathrm{emp}} \| P_{\theta^*}) > \epsilon.
\end{equation}

Combining with \eqref{eq:kl_approx} and \eqref{eq:kl_emp_zero}:
\begin{equation}
d_{\mathrm{KL}}(P_{\mathrm{emp}} \| P_{\theta^*}) \leqslant \epsilon,
\end{equation}
which contradicts \eqref{eq:kl_final_lower}. Therefore, the assumption that mode collapse does not occur is false.

Hence, there exists $S \subset \mathcal{X}$ with $|S| = K$ such that
\begin{equation}
P_{\theta^*}(S) \geqslant 1 - \epsilon.
\end{equation}

Choosing $S = S_0$, the empirical support, we have
\begin{equation}
P_{\theta^*}(S_0) \geqslant 1 - \epsilon,
\end{equation}
which shows that the generated probability mass concentrates on the empirical support, i.e., strict mode collapse occurs.

With these analyses in place, the following theorem formalizes the emergence of mode collapse.

\begin{theorem}[Mode collapse under stable training]
\label{thm:mode_collapse}
Let $\theta^*$ be a stationary point of a stable training trajectory.  
Then there exists a set $S \subset \mathcal{X}$ with $|S| = K \ll |\mathcal{X}|$ such that
\begin{equation}
P_{\theta^*}(S) \geqslant 1 - \epsilon,
\end{equation}
where $\epsilon = \delta = \lambda M \to 0^+$.  
In other words, $P_{\theta^*}$ concentrates almost all probability mass on the empirical support, exhibiting strict mode collapse.
\end{theorem}

\section{Experiments}

\subsection{Experimental Setup}

\begin{figure*}[htbp]
\centering
\includegraphics[width=0.85\linewidth]{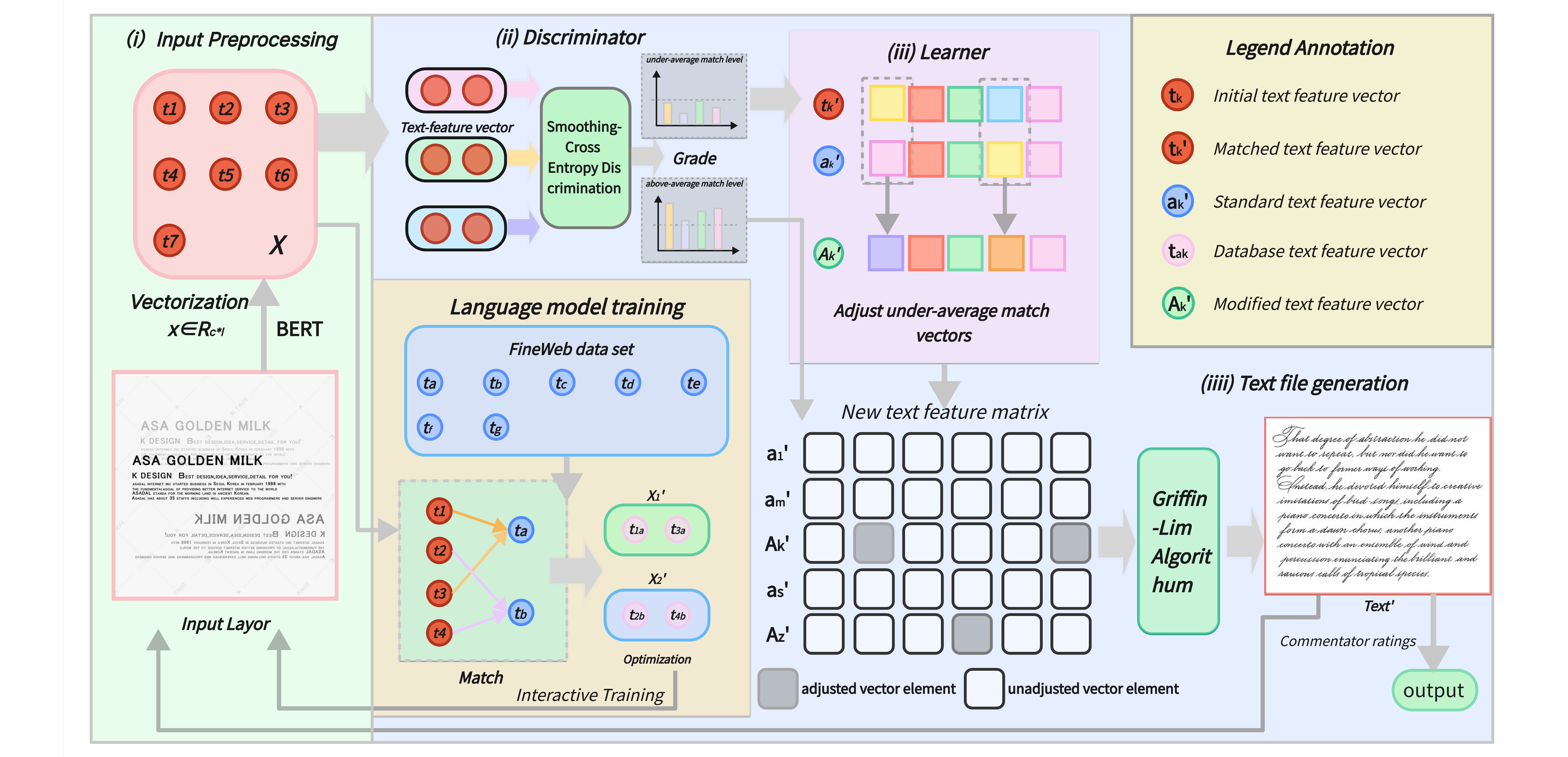}
\caption{Framework for studying mode collapse: feedback network with BARRA module. This closed-loop control system stabilizes large language model generation dynamics, integrating stochastic modules to regulate learning intensity and inject controlled variability. It isolates stability-induced mode collapse by decoupling dynamics stabilization from semantic optimization, capturing key degenerative behaviors like low-entropy outputs and repetitive token sequences.}
\label{fig:exp_setup}
\end{figure*}

%%%~\citep{zheng2024evaluating}
%%%%下面引用
To empirically validate Theorem~\ref{thm:mode_collapse}, we construct a closed-loop control network tailored to probe stability-induced mode collapse. The network implements a first-order negative feedback system with discrete-time dynamics
\begin{equation}
x_{t+1} = x_t - \alpha f(x_t) + \eta_t,
\end{equation}
where $\alpha>0$ is the step size, $f$ encodes the feedback function, and $\eta_t$ is stochastic noise. Learning is stabilized via our PSU algorithm, which biases parameter trajectories toward asymptotic boundedness and vanishing gradients, satisfying the stable training assumptions of Assumption \ref{ass5.1}. This configuration ensures the system dynamics are tractable while allowing direct observation of emergent behavior under stable training.

To systematically explore the network's output variability, we integrate two stochastic modules: The Dynamic MLP Unit (DMU) randomly drops activations at each step to regulate learning intensity and enforce sparsity of active modes, while the Dynamic Transformer Unit (DTU) injects stochastic perturbations into parameter updates, simulating variability in learning paths. Together, these modules maintain controlled exploration within the stable training regime, enabling the empirical distribution of outputs to concentrate on specific modes. The resulting probability concentration is monitored over time to quantify mode collapse. Figure~\ref{fig:exp_setup} provides a schematic overview of the system. Our detailed impletation and analysis of experiment can be seen in Appendix~\ref{sec:definitions} and ~\ref{sec:baara}.

\subsection{Main Experiments}

We empirically validate stable training induces mode collapse through forward KL minimization and generative entropy reduction. Our experiments focus on feedback strength $\alpha$ as a direct regulator of internal generation stability. High-frequency word proportion serves as our core metric for probability mass concentration. Higher values indicate stricter confinement to narrow empirical modes and align with mode collapse's theoretical definition.

\begin{table}[htbp]
  \centering
  \caption{High-Frequency Word Proportion}
  \label{tab:main_exp}
  \definecolor{gray-bg}{gray}{0.95}
  \begin{tabular}{lccc}
    \toprule
    Condition & $\alpha$=0 & $\alpha$=0.05 & $\alpha$=0.2 \\
    \midrule
    GPT-2,seed=2025,1B & 27.5\% & 67.8\% & 100\% \\
    \rowcolor{gray-bg}
    GPT-2,seed=2026,1B & 26.7\% & 61.2\% & 100\% \\
    GPT-2,seed=2027,1B & 26.3\% & 68.3\% & 100\% \\
    \rowcolor{gray-bg}
    BERT,seed=2026,1B & 23.5\% & 56.3\% & 100\% \\
    GPT-2,seed=2026,2B & 24.2\% & 58.2\% & 100\% \\
    \rowcolor{gray-bg}
    GPT-2,seed=2026,3B & 21.9\% & 56.1\% & 100\% \\
    \bottomrule
  \end{tabular}
\end{table}

Results confirm a deterministic relationship between stabilization intensity and mode collapse. This monotonic trend verifies our theorem that stable training pushes models toward forward KL minimization and probability mass concentration on limited empirical modes.

Varied GPT-2 parameter scales rule out overfitting. Larger models exhibit slightly lower high-frequency word proportions at $\alpha=0.05$. This contradicts the overfitting hypothesis and confirms mode collapse stems from stabilization-induced parameter constraints not insufficient capacity. Cross-architecture validation with GPT-2 and BERT demonstrates universality. Both follow identical response patterns to $\alpha$ which proves stability-induced mode collapse is a fundamental MLE training property unrelated to architectural design.

Notably perplexity rises not from poor optimization. Training loss converges smoothly across all $\alpha$. The increase comes from the inherent tradeoff between stabilization-induced forward KL minimization and natural language's structural complexity. Stabilization constraints force the model to concentrate on high-frequency tokens. This erodes access to semantically critical low-probability transitions. Diminished generative entropy reduces adaptability to contextual nuances. It leads to repetitive sequences that violate linguistic conditional dependencies. Ultimately optimization stability masks linguistic competence breakdown. Smoother convergence directly correlates with reduced diversity and increased perplexity.

\subsection{Ablation Studies}

\begin{figure}[htbp]
    \centering
    \includegraphics[width=0.45\textwidth]{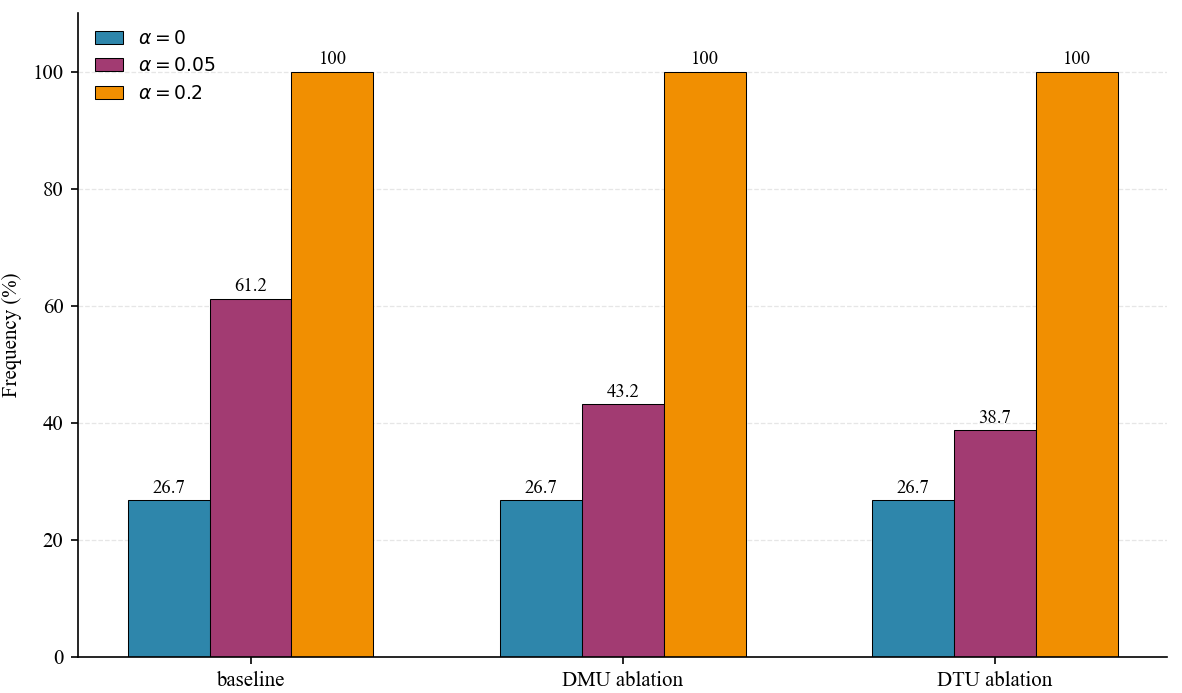}
    \caption{Ablation Studies of BARRA's DMU/DTU Components. High-frequency word proportion across $\alpha$ values quantifies synergistic stabilization effects on mode collapse via KL minimization and entropy reduction. Results show DMU and DTU jointly intensify mode concentration underperforming single-module ablations.}
    \label{fig:ablation}
\end{figure}

We conduct ablation studies on BARRA’s DMU and DTU to validate that stabilization induces mode collapse via forward KL minimization and entropy reduction, as shown in Figure.~\ref{fig:ablation}. Removing each component and comparing to the full BARRA baseline we quantify how the two components synergistically amplify stabilization effects worsening the entropy-loss trajectory driving mode collapse. All configurations share identical setup ensuring performance differences reflect BARRA’s stabilization impact on the tradeoff between optimization stability and generative degradation.

At $\alpha=0$, all configurations show {26.7}\% high-frequency word proportion, confirming stabilization triggers mode collapse. At $\alpha=0.05$, the baseline reaches {61.2}\% high-frequency word proportion versus {43.2}\% in DMU ablation and {38.7}\% in DTU ablation; this gradient shows BARRA components synergistically boost stabilization-induced mode concentration. DTU tightens logit constraints to accelerate forward KL minimization, DMU prunes low-probability modes to reduce entropy, and their combined action intensifies the entropy-loss trajectory, pushing the model closer to mode collapse. At $\alpha=0.2$, all hit {100}\% high-frequency word proportion, but the baseline converges {24}\% faster than DTU ablation, demonstrating BARRA's stabilized training accelerates inevitable mode collapse.

These results validate our theory that stability is a liability. Differential performance at moderate $\alpha$ links BARRA's components to forward KL minimization and entropy reduction, tying stabilization to generative degradation. Conditional activation of BARRA components underscores that mode collapse is stability-dependent and rooted in MLE properties. Accelerated collapse in the baseline at extreme $\alpha$ confirms stabilization optimizes optimization dynamics at the cost of generative expressivity.

\subsection{Loss-Entropy Trajectory and Mode Collapse Vulnerability}

\begin{figure}[htbp]
    \centering
    \includegraphics[width=0.5\textwidth]{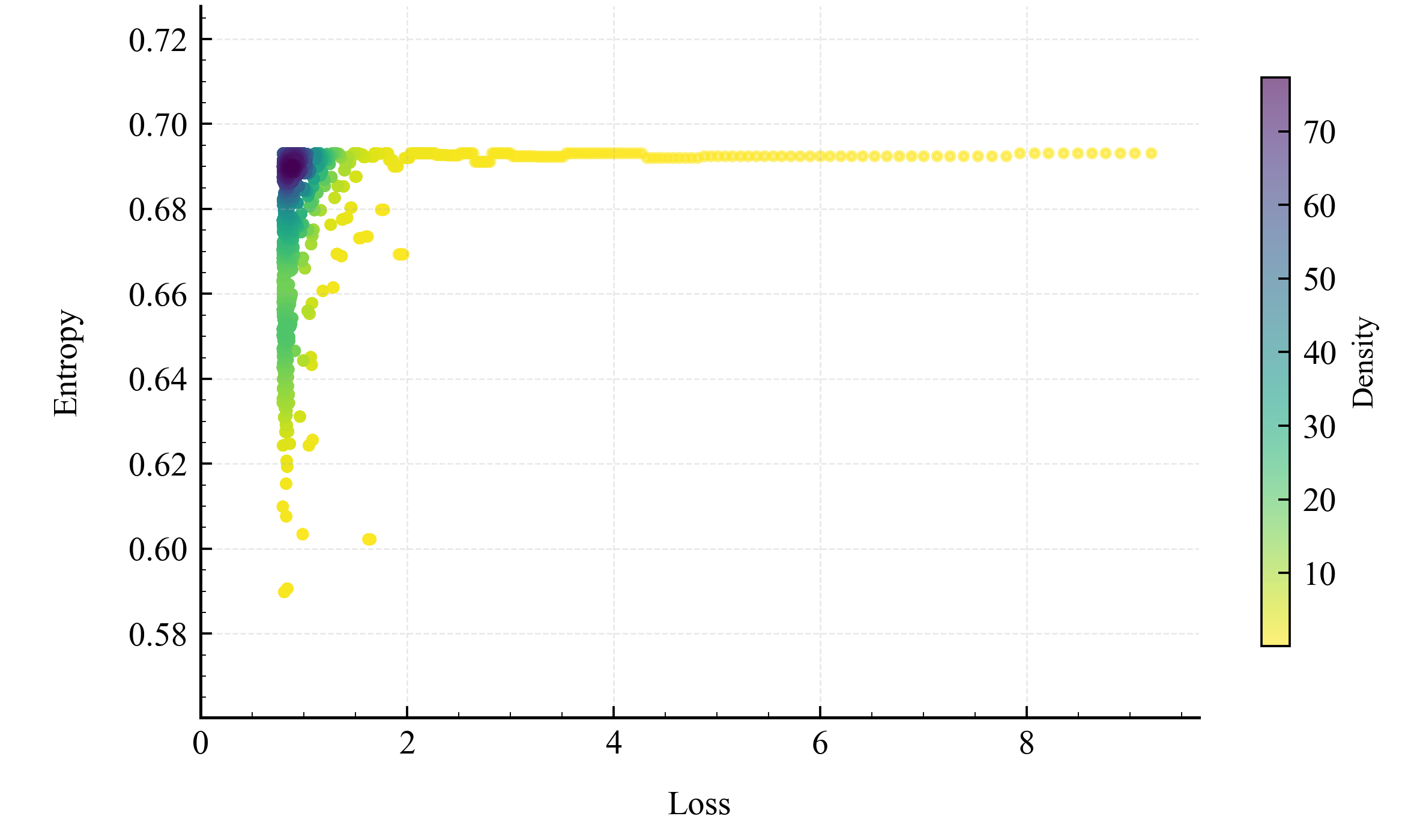}
    \caption{Joint Trajectory of Training Loss and Generative Entropy in LLM Training. Color intensity marks point density, revealing the tradeoff between loss minimization and generative diversity across training steps. High-loss stages show sharp entropy drops while low-loss regimes yield marginal entropy reductions, reflecting pathological probability concentration.}
    \label{fig:loss_entropy}
\end{figure}

To further validate how optimization dynamics drive mode collapse we analyze the joint trajectory of training loss and generative entropy, shown in Figure.\ref{fig:loss_entropy}. This relationship reveals the inherent vulnerability of low-loss regimes to mode collapse.

The plot tracks loss against entropy across training steps with color intensity marking point density. In high loss states entropy falls sharply as loss decreases small reductions in loss correspond to large entropy drops as the model rapidly converges to dominant empirical modes concentrating probability mass early in training.

In low loss states the entropy loss relationship flattens substantially even large decreases in loss yield only marginal entropy reductions. The model has exhausted most degrees of freedom for entropy reduction through broad pattern learning it now can only further minimize loss by stripping away remaining low probability yet semantically critical modes that sustain generative diversity. Narrowing probability mass to an ever smaller subset of tokens is the only available path to lower loss in late stage training.

The density heatmap underscores this pattern most training time is spent in the low loss low entropy region. The model remains trapped in minimal entropy states despite continued loss convergence because the MLE objective incentivizes prioritizing fit to dominant modes over preserving full diversity of the empirical distribution.

This trajectory exposes the fundamental cause of mode collapse. Diminishing entropy returns at low loss mean late stage training enforces pathological concentration of probability mass. The model abandons rare semantically diverse tokens because the optimization objective rewards convergence to the narrowest possible set of modes that minimize loss. Mode collapse thus becomes an inevitable outcome of pushing loss to extreme lows rather than a failure of training optimization.
\section{Conclusion}

In this work we establish a rigorous connection between stable training trajectories and mode collapse in generative models. We formalize conditions where stationary points of stabilized MLE objectives concentrate probability mass on a small subset of empirical support and verify these predictions in a controlled closed-loop feedback system with stochastic exploration modules. Our analysis clarifies the fundamental tension between stability and diversity in high-dimensional learning providing a principled framework to anticipate measure and potentially mitigate mode collapse. Theoretical guarantees and empirical validation across architectures and parameter scales highlight the importance of parameter dynamics in designing robust generative systems. This work lays the foundation for future explorations into balancing stability expressivity and controlled stochasticity in generative modeling.

\section*{Acknowledgements}
This work was supported in part by the National Natural Science Foundation of China under Grants 62373158; in part by the State Key Laboratory of Mechanical Transmission for Advanced Equipment under Grant SKLMT-MSKFKT-202317; and in part by the Open Research Project of the State Key Laboratory of Industrial Control Technology, China , under Grant ICT2024B22. The computation is
completed in the HPC Platform of Huazhong University of Science and Technology. We also thank Yibo Zhong from Tsinghua University to help us check our mathematical proof.

\section*{Impact Statement}

Our research uncovers the tradeoff between training stability and generative expressivity in large language models, offering critical insights for robust model development. The findings guide the design of balanced training frameworks that avoid stability-induced degeneration, with potential applications in content generation, dialogue systems, and other high-stakes NLP tasks. We confirm that all experimental setups and datasets adhere to ethical standards, and the research itself does not involve sensitive content or pose direct ethical risks. Future work leveraging these insights should continue to prioritize the alignment of optimization stability with real-world utility, ensuring generative models retain diversity and reliability in practical use.
\nocite{langley00}

\bibliography{cite}

@inproceedings{koloskova2023revisiting,
  title={Revisiting gradient clipping: Stochastic bias and tight convergence guarantees},
  author={Koloskova, Anastasia and Hendrikx, Hadrien and Stich, Sebastian U},
  booktitle={International Conference on Machine Learning},
  pages={17343--17363},
  year={2023},
  organization={PMLR}
}

@article{chiarella2006dynamic,
  title={A dynamic analysis of moving average rules},
  author={Chiarella, Carl and He, Xue-Zhong and Hommes, Cars},
  journal={Journal of Economic Dynamics and Control},
  volume={30},
  number={9-10},
  pages={1729--1753},
  year={2006},
  publisher={Elsevier}
}

@article{girosi1995regularization,
  title={Regularization theory and neural networks architectures},
  author={Girosi, Federico and Jones, Michael and Poggio, Tomaso},
  journal={Neural computation},
  volume={7},
  number={2},
  pages={219--269},
  year={1995},
  publisher={MIT Press}
}

@article{zhao2022neural,
  title={Neural-network-based adaptive finite-time output feedback control for spacecraft attitude tracking},
  author={Zhao, Lin and Yu, Jinpeng and Chen, Xinkai},
  journal={IEEE Transactions on Neural Networks and Learning Systems},
  volume={34},
  number={10},
  pages={8116--8123},
  year={2022},
  publisher={IEEE}
}

@article{lyu2019advances,
  title={Advances in neural information processing systems},
  author={Lyu, He and Sha, Ningyu and Qin, Shuyang and Yan, Ming and Xie, Yuying and Wang, Rongrong},
  journal={Advances in neural information processing systems},
  volume={32},
  year={2019}
}

@misc{learning11r1,
      title={R1-Reward: Training Multimodal Reward Model Through Stable Reinforcement Learning}, 
      author={Yi-Fan Zhang and Xingyu Lu and Xiao Hu and Chaoyou Fu and Bin Wen and Tianke Zhang and Changyi Liu and Kaiyu Jiang and Kaibing Chen and Kaiyu Tang and Haojie Ding and Jiankang Chen and Fan Yang and Zhang Zhang and Tingting Gao and Liang Wang},
      year={2025},
      eprint={2505.02835},
      archivePrefix={arXiv},
      primaryClass={cs.CV},
}

@article{gumen2013dynamically,
  title={Dynamically stable preferences},
  author={Gumen, Anna and Savochkin, Andrei},
  journal={Journal of Economic Theory},
  volume={148},
  number={4},
  pages={1487--1508},
  year={2013},
  publisher={Elsevier}
}

@article{akhtar2024roboss,
  title={RoBoSS: A robust, bounded, sparse, and smooth loss function for supervised learning},
  author={Akhtar, Mushir and Tanveer, M and Arshad, Mohd},
  journal={IEEE Transactions on Pattern Analysis and Machine Intelligence},
  year={2024},
  publisher={IEEE}
}

@inproceedings{zheng2016improving,
  title={Improving the robustness of deep neural networks via stability training},
  author={Zheng, Stephan and Song, Yang and Leung, Thomas and Goodfellow, Ian},
  booktitle={Proceedings of the ieee conference on computer vision and pattern recognition},
  pages={4480--4488},
  year={2016}
}

@article{shaham2018understanding,
  title={Understanding adversarial training: Increasing local stability of supervised models through robust optimization},
  author={Shaham, Uri and Yamada, Yutaro and Negahban, Sahand},
  journal={Neurocomputing},
  volume={307},
  pages={195--204},
  year={2018},
  publisher={Elsevier}
}

@article{rybakov2024methods,
  title={Methods of improving llm training stability},
  author={Rybakov, Oleg and Chrzanowski, Mike and Dykas, Peter and Xue, Jinze and Lanir, Ben},
  journal={arXiv preprint arXiv:2410.16682},
  year={2024}
}

@article{tu2025survey,
  title={A Survey on LLM Mid-Training},
  author={Tu, Chengying and Zhang, Xuemiao and Weng, Rongxiang and Li, Rumei and Zhang, Chen and Bai, Yang and Yan, Hongfei and Wang, Jingang and Cai, Xunliang},
  journal={arXiv preprint arXiv:2510.23081},
  year={2025}
}

@article{wei2025control,
  title={Control LLM: Controlled Evolution for Intelligence Retention in LLM},
  author={Wei, Haichao and Ren, Yunxiang and Fu, Zhoutong and Lunia, Aman and Chen, Yi-Lin and Leung, Alice and Xu, Ya},
  journal={arXiv preprint arXiv:2501.10979},
  year={2025}
}

@article{liu2020understanding,
  title={Understanding the difficulty of training transformers},
  author={Liu, Liyuan and Liu, Xiaodong and Gao, Jianfeng and Chen, Weizhu and Han, Jiawei},
  journal={arXiv preprint arXiv:2004.08249},
  year={2020}
}

@article{cui2025entropy,
  title={The entropy mechanism of reinforcement learning for reasoning language models},
  author={Cui, Ganqu and Zhang, Yuchen and Chen, Jiacheng and Yuan, Lifan and Wang, Zhi and Zuo, Yuxin and Li, Haozhan and Fan, Yuchen and Chen, Huayu and Chen, Weize and others},
  journal={arXiv preprint arXiv:2505.22617},
  year={2025}
}

@article{agarwal2025unreasonable,
  title={The unreasonable effectiveness of entropy minimization in llm reasoning},
  author={Agarwal, Shivam and Zhang, Zimin and Yuan, Lifan and Han, Jiawei and Peng, Hao},
  journal={arXiv preprint arXiv:2505.15134},
  year={2025}
}

@inproceedings{zhou2020improving,
  title={Improving Autoregressive NMT with Non-Autoregressive Model},
  author={Zhou, Long and Zhang, Jiajun and Zong, Chengqing},
  booktitle={Proceedings of the First Workshop on Automatic Simultaneous Translation},
  pages={24--29},
  year={2020}
}

@inproceedings{kulikov2022characterizing,
  title={Characterizing and addressing the issue of oversmoothing in neural autoregressive sequence modeling},
  author={Kulikov, Ilia and Eremeev, Maksim and Cho, Kyunghyun},
  booktitle={Proceedings of the 2nd Conference of the Asia-Pacific Chapter of the Association for Computational Linguistics and the 12th International Joint Conference on Natural Language Processing (Volume 1: Long Papers)},
  pages={1115--1124},
  year={2022}
}

@article{papamakarios2021normalizing,
  title={Normalizing flows for probabilistic modeling and inference},
  author={Papamakarios, George and Nalisnick, Eric and Rezende, Danilo Jimenez and Mohamed, Shakir and Lakshminarayanan, Balaji},
  journal={Journal of Machine Learning Research},
  volume={22},
  number={57},
  pages={1--64},
  year={2021}
}

@article{andreyev2024edge,
  title={Edge of stochastic stability: Revisiting the edge of stability for sgd},
  author={Andreyev, Arseniy and Beneventano, Pierfrancesco},
  journal={arXiv preprint arXiv:2412.20553},
  year={2024}
}

@article{qi2018slope,
  title={Slope stability prediction using integrated metaheuristic and machine learning approaches: A comparative study},
  author={Qi, Chongchong and Tang, Xiaolin},
  journal={Computers \& Industrial Engineering},
  volume={118},
  pages={112--122},
  year={2018},
  publisher={Elsevier}
}

@article{amroune2021machine,
  title={Machine Learning Techniques Applied to On-Line Voltage Stability Assessment: A Review.},
  author={Amroune, Mohammed},
  journal={Archives of Computational Methods in Engineering},
  volume={28},
  number={2},
  year={2021}
}

@phdthesis{goyal2021characterizing,
  title={Characterizing and Overcoming the Limitations of Neural Autoregressive Models},
  author={Goyal, Kartik},
  year={2021},
  school={Carnegie Mellon University}
}

@incollection{pan2002maximum,
  title={Maximum likelihood estimation},
  author={Pan, Jian-Xin and Fang, Kai-Tai},
  booktitle={Growth curve models and statistical diagnostics},
  pages={77--158},
  year={2002},
  publisher={Springer}
}

@article{van2014renyi,
  title={R{\'e}nyi divergence and Kullback-Leibler divergence},
  author={Van Erven, Tim and Harremos, Peter},
  journal={IEEE Transactions on Information Theory},
  volume={60},
  number={7},
  pages={3797--3820},
  year={2014},
  publisher={IEEE}
}

@article{das2024under,
  title={Under the surface: Tracking the artifactuality of llm-generated data},
  author={Das, Debarati and De Langis, Karin and Martin-Boyle, Anna and Kim, Jaehyung and Lee, Minhwa and Kim, Zae Myung and Hayati, Shirley Anugrah and Owan, Risako and Hu, Bin and Parkar, Ritik and others},
  journal={arXiv preprint arXiv:2401.14698},
  year={2024}
}

@article{susanti2025can,
  title={Can llms leverage observational data? towards data-driven causal discovery with llms},
  author={Susanti, Yuni and F{\"a}rber, Michael},
  journal={arXiv preprint arXiv:2504.10936},
  year={2025}
}

@article{zhu2025llm,
  title={Where llm agents fail and how they can learn from failures},
  author={Zhu, Kunlun and Liu, Zijia and Li, Bingxuan and Tian, Muxin and Yang, Yingxuan and Zhang, Jiaxun and Han, Pengrui and Xie, Qipeng and Cui, Fuyang and Zhang, Weijia and others},
  journal={arXiv preprint arXiv:2509.25370},
  year={2025}
}

@inproceedings{chow2024performance,
  title={Performance optimization in the llm world 2024},
  author={Chow, Kingsum and Tang, Yu and Lyu, Zhiheng and Rajput, Anil and Ban, Khun},
  booktitle={Companion of the 15th ACM/SPEC International Conference on Performance Engineering},
  pages={156--157},
  year={2024}
}

@article{van2024loop,
  title={In-the-loop hyper-parameter optimization for llm-based automated design of heuristics},
  author={van Stein, Niki and Vermetten, Diederick and B{\"a}ck, Thomas},
  journal={ACM Transactions on Evolutionary Learning},
  year={2024},
  publisher={ACM New York, NY}
}

@article{yin2024entropy,
  title={Entropy law: The story behind data compression and llm performance},
  author={Yin, Mingjia and Wu, Chuhan and Wang, Yufei and Wang, Hao and Guo, Wei and Wang, Yasheng and Liu, Yong and Tang, Ruiming and Lian, Defu and Chen, Enhong},
  journal={arXiv preprint arXiv:2407.06645},
  year={2024}
}

@article{chen2020understanding,
  title={Understanding gradient clipping in private sgd: A geometric perspective},
  author={Chen, Xiangyi and Wu, Steven Z and Hong, Mingyi},
  journal={Advances in Neural Information Processing Systems},
  volume={33},
  pages={13773--13782},
  year={2020}
}

@inproceedings{qian2021understanding,
  title={Understanding gradient clipping in incremental gradient methods},
  author={Qian, Jiang and Wu, Yuren and Zhuang, Bojin and Wang, Shaojun and Xiao, Jing},
  booktitle={International Conference on Artificial Intelligence and Statistics},
  pages={1504--1512},
  year={2021},
  organization={PMLR}
}

@inproceedings{cabello2023impact,
  title={The impact of data normalization on the accuracy of machine learning algorithms: A comparative analysis},
  author={Cabello-Solorzano, Kelsy and Ortigosa de Araujo, Isabela and Pe{\~n}a, Marco and Correia, Lu{\'\i}s and J. Tall{\'o}n-Ballesteros, Antonio},
  booktitle={International conference on soft computing models in industrial and environmental applications},
  pages={344--353},
  year={2023},
  organization={Springer}
}

@article{bharathi2023text,
  title={Text summarization for big data analytics: a comprehensive review of GPT 2 and BERT approaches},
  author={Bharathi Mohan, G and Prasanna Kumar, R and Parathasarathy, Srinivasan and Aravind, S and Hanish, KB and Pavithria, G},
  journal={Data analytics for internet of things infrastructure},
  pages={247--264},
  year={2023},
  publisher={Springer}
}

@article{bahani2023effectiveness,
  title={The effectiveness of T5, GPT-2, and BERT on text-to-image generation task},
  author={Bahani, Mourad and El Ouaazizi, Aziza and Maalmi, Khalil},
  journal={Pattern recognition letters},
  volume={173},
  pages={57--63},
  year={2023},
  publisher={Elsevier}
}

@article{lima2023large,
  title={A large comparison of normalization methods on time series},
  author={Lima, Felipe Tomazelli and Souza, Vinicius MA},
  journal={Big Data Research},
  volume={34},
  pages={100407},
  year={2023},
  publisher={Elsevier}
}
\bibliographystyle{icml2026}
\newpage
\appendix
\onecolumn
% 第一部分：定义

\section{Definition and Notations}

\label{sec:definitions}
\begin{definition}[Learning Accuracy $\widetilde{a}$]
For a batch $x_b$ and model parameters $\theta$:

\begin{equation}
\widetilde{a}(x_b, \theta) = \frac{|\nabla_\theta \mathcal{L}(\theta; x_b)|}{\max_{\theta'} |\nabla_{\theta'} \mathcal{L}(\theta'; x_b)|},
\end{equation}

where $\mathcal{L}(\theta;x_b)$ is the loss function.
\end{definition}

\textbf{Intuition:} The learning accuracy $\widetilde{a} \in [0,1]$ quantifies the expected contribution of a batch to effective learning. It measures the information content of a batch, scaling proportionally with the learning rate $\eta$ in practice.

\begin{definition}[Learning Complexity $\widetilde{c}$]
For a batch $x_b$:

\begin{equation}
\widetilde{c}(x_b) = \frac{T_{\text{actual}}(x_b)}{T_{\text{full}}(x_b)},
\end{equation}

where $T_{\text{actual}}$ is the actual computation time with dropout and $T_{\text{full}}$ is the full computation time without dropout.
\end{definition}

\textbf{Intuition:} The learning complexity $\widetilde{c} \in [0,1]$ measures the computational cost required for processing a batch. It satisfies $\widetilde{c} = \Theta(t)$, where $t$ denotes the associated training time.

\begin{definition}[Edge-level Dropout]
Edge-level dropout applies dropout independently to each edge in a neural network. For a weight matrix $W \in \mathbb{R}^{n \times m}$, each element $w_{ij}$ is independently retained with probability.
\end{definition}

\begin{definition}[Node-level Dropout]
Node-level dropout applies dropout to entire nodes in a neural network. For a layer with $n$ nodes, each node is independently retained with probability $\widetilde{c}$, which affects all edges connected to that node.
\end{definition}

\section{Batch-Aware Adaptive Resource Allocation (BAARA)}
\label{sec:baara}

Derived directly from the need to address batch heterogeneity, we designed Batch-Aware Adaptive Resource Allocation (BAARA). BAARA is a closed-loop discriminator-learner framework that dynamically allocates computational resources based on batch-specific characteristics.

\subsection{Core Design Rationale}
\label{sec:design_rationale}

The Batch-Aware Active Training Framework is based on a clear mathematical principle: in heterogeneous batch training, optimal updates scale the learning rate with the ratio of a batch’s information value $\widetilde{a}$ to its computational cost $\widetilde{c}$:
\begin{equation}
\eta \propto \frac{\widetilde{a}}{\widetilde{c}}, \;
\eta_{\text{mod}} = \frac{2}{\pi} \arctan\left(\frac{\widetilde{a}}{\widetilde{c}}\right) \eta_0.
\end{equation}

To operationalize this, BAARA represents each batch using gradient norms and loss dynamics and token rarity plus layer-wise Fisher information. This unified state allows accurate computation of the ${\widetilde{a}}/{\widetilde{c}}$ ratio for each batch.

Updates are controlled along two dimensions: \textbf{update intensity} $\gamma_b$ scales the effective learning rate with $\widetilde{a}$, and \textbf{update scope} $\mathcal{S}_b$ restricts parameter updates to an active subspace, scaling computation with $\widetilde{c}$. Together, these ensure strong updates for high-value, low-cost batches and conservative updates otherwise.

\subsection{Decision-making Layer}
\label{sec:algorithm}

BAARA uses a lightweight decision-making layer $\mathcal{D}_\phi$ to map batch states to resource allocation actions, trained via reinforcement learning to maximize long-term training efficiency. The full algorithm is formalized below.

\begin{algorithm*}
    \caption{Predictive Selective Update (PSU)}
    \label{alg:psu}
    \begin{algorithmic}[1]
        \REQUIRE Model parameters $\theta_0$, discriminator $\mathcal{D}_\phi$, feature extractor, training batches $\{x_b\}$
        \FOR{each batch $x_b$}
            \STATE Extract state features $s_b = x_b^\text{feat}$
            \STATE Sample action $a_b = (\gamma_b, \mathcal{S}_b) \sim \pi_\phi(s_b)$
            \STATE Update learner: $\theta \gets \theta - \gamma_b \pi_\phi\,\nabla_\theta \mathcal{L}(\theta;x_b)$
            \STATE Measure batch execution cost $T_b^\star$ and compute reward $r_b$
            \STATE Store $(s_b, T_b^\star, r_b)$ in replay buffer
            \STATE Periodically update $\phi$ by maximizing expected reward $\mathcal{L}_D$
        \ENDFOR
    \end{algorithmic}
\end{algorithm*}

\paragraph{State.}
Each batch $b$ is represented by a compact feature vector 
$s_b = x_b^\text{feat}$, 
constructed from gradient statistics, loss dynamics, token rarity, and optional layer-wise Fisher information. These features characterize both batch informativeness and model sensitivity.

\paragraph{Action.}
The action $a_b = (\gamma_b, \mathcal{S}_b)$ consists of:
(i) a precision factor $\gamma_b$ modulating the effective learning rate, and
(ii) an active subspace $\mathcal S_b \subseteq \{1,\ldots,d\}$ specifying which coordinates to update.
This enables non-uniform and selective parameter updates.

\paragraph{Reward.}
The instantaneous contribution of batch $B_t$ is measured by the raw batch value
\begin{equation}
    V(B_t) = \mathcal{L}(\theta_t; B_t) - \mathcal{L}(\theta_{t+1}; B_t).
\end{equation}
To reduce noise sensitivity, we compute a smoothed value via a window of size $\Delta$:
\begin{equation}
    \bar{V}_t = \frac{1}{\Delta}\sum_{i=t-\Delta+1}^{t} V(B_i).
\end{equation}
Given predictive entropy
\begin{equation}
    H(B_t) = -\frac{1}{|B_t|} \sum_{i \in B_t} \sum_j p_{i,j} \log p_{i,j},
\end{equation}
the confidence factor is defined as $C_t = \exp(-\alpha H(B_t))$.
The final reward is
\begin{equation}
    r_t = \bar{V}_t\,C_t - \lambda (s_t - 1)^2.
\end{equation}

\paragraph{Update.}This layer serves as a stochastic policy $\pi_\phi(a_b \mid s_b)$.
The learner then performs a projected gradient step:
\begin{equation}
    \theta \gets \theta - \gamma_b P_{\mathcal S_b}\nabla_\theta \mathcal{L}(\theta; x_b),
\end{equation}
where $P_{\mathcal S_b}$ denotes the projection onto the active subspace.

For a purely noisy batch, assume
\begin{equation}
\mathbb{E}[\nabla_\theta \mathcal{L}(\theta_t; B_t)] = 0 \implies \mathbb{E}[V(B_t)] \approx 0.
\end{equation}
Moreover, high predictive entropy implies
\begin{equation}
C_t \approx \exp(-\alpha H(B_t)) \ll 1.
\end{equation}
Hence,
\begin{equation}
\mathbb{E}[r_t] = \mathbb{E}[\bar{V}_t \cdot C_t - \lambda (s_t - 1)^2] \leqslant 0,
\end{equation}
showing that the update intensity $s_t$ is naturally discouraged for noisy batches.

For a purely noisy batch, assume
\begin{equation}
\mathbb{E}[\nabla_\theta \mathcal{L}(\theta_t; B_t)] = 0 \implies \mathbb{E}[V(B_t)] \approx 0.
\end{equation}
Moreover, high predictive entropy implies
\begin{equation}
C_t \approx \exp(-\alpha H(B_t)) \ll 1.
\end{equation}
Hence,
\begin{equation}
\mathbb{E}[r_t] = \mathbb{E}[\bar{V}_t \cdot C_t - \lambda (s_t - 1)^2] \leqslant 0,
\end{equation}
showing that the update intensity $s_t$ is naturally discouraged for noisy batches.

The smoothed batch value satisfies
\begin{equation}
\bar{V}_t = \frac{1}{\Delta} \sum_{i=t-\Delta+1}^{t} V(B_i) = \frac{V(B_t)}{\Delta} + \frac{1}{\Delta}\sum_{i=t-\Delta+1}^{t-1} V(B_i),
\end{equation}
so that
\begin{equation}
\begin{split}
|\bar{V}_t - V(B_t)| &= \frac{1}{\Delta} \left| \sum_{i=t-\Delta+1}^{t-1} (V(B_i) - V(B_t)) \right|\\
&\leqslant \frac{\Delta-1}{\Delta} \max_{i \in [t-\Delta+1,t]} |V(B_i) - V(B_t)|.
\end{split}
\end{equation}
This shows that transient fluctuations are bounded, preserving stable updates.

Combining the above, we have the following theorem. This behavior ensures that update intensity automatically concentrates on informative, low-uncertainty batches while suppressing updates on uninformative or noisy batches.

\begin{theorem}[Adaptive Update Scaling in Heterogeneous Batches]
\label{thm:adaptive_scaling}
Consider a learner updated by a stochastic policy $\pi_\phi(a_b \mid s_b)$ that receives rewards based on batch value and predictive confidence. Then, under general conditions of batch heterogeneity and reward smoothing, the learned batch scaling factor $s_t$ exhibits the following natural adaptation:
\begin{enumerate}
\item $s_t \to 1 \quad \text{for batches that are noisy or of high entropy.}$
\item $s_t \uparrow \text{ for batches with high learning value and low entropy.}$
\end{enumerate}
\end{theorem}

\subsection{Action Layer}
\subsubsection{Primary Algorithms for MLP}

\begin{algorithm}[h]
    \caption{DMU2}
    \label{alg:dmu2}
    \begin{algorithmic}[1]
        \REQUIRE $W \in \mathbb{R}^{n \times m}$, $\vec{b} \in \mathbb{R}^n$, $\widetilde{c} \in (0,1)$, mode $\in \{\text{Prob}, \text{Fix}\}$
        \ENSURE Masked weights $W_{\text{mask}}$, masked biases $\vec{b}_{\text{mask}}$
        \IF{$\text{mode} = \text{Prob}$}
            \STATE $\eta_{ij} \xrightarrow{i.i.d.} \text{Bernoulli}(\widetilde{c})$
            \STATE $\eta_i^b \xrightarrow{i.i.d.} \text{Bernoulli}(\widetilde{c})$
            \STATE $W_{\text{mask},ij} = \eta_{ij} W_{ij}$, $\vec{b}_{\text{mask},i} = \eta_i^b \vec{b}_i$
        \ELSIF{$\text{mode} = \text{Fix}$}
            \STATE $k_W = \lfloor{\widetilde{c} nm}\rfloor$, $\mathcal{S}_W \sim \text{Uniform}\binom{\lfloor{n}\rfloor \times \lfloor m\rfloor}{k_W}$
            \STATE $k_b = \lfloor{\widetilde{c} n}\rfloor$, $\mathcal{S}_b \sim \text{Uniform}\binom{[n]}{k_b}$
            \STATE $W_{\text{mask},ij} = W_{ij} \mathbb{I}{(i,j) \in \mathcal{S}_W}$, $\vec{b}_{\text{mask},i} = \vec{b}_i \mathbb{I}{(i)\in \mathcal{S}_b}$
        \ENDIF
        \STATE Return $W_{\text{mask}}$, $\vec{b}_{\text{mask}}$
    \end{algorithmic}
\end{algorithm}

\begin{algorithm}[h]
    \caption{Fixed-proportion Node Retention (DFix)}
    \label{alg:dmu_trans}
    \begin{algorithmic}[1]
        \REQUIRE $W \in \mathbb{R}^{n \times m}$, $\vec{b} \in \mathbb{R}^n$, $\vec{x} \in \mathbb{R}^m$, $\widetilde{c} \in (0,1)$
        \ENSURE Masked output $\vec{z}_{\text{mask}}$, update masks $M_W, M_b$
        \STATE $k_{\text{in}} = \lfloor{\widetilde{c} \cdot m}\rfloor$, $\mathcal{S}_{\text{in}} \sim \text{Uniform}\binom{\lfloor m\rfloor}{k_{\text{in}}}$
        \STATE $k_{\text{out}} = \lfloor{\widetilde{c} \cdot n}\rfloor$, $\mathcal{S}_{\text{out}} \sim \text{Uniform}\binom{\lfloor n\rfloor}{k_{\text{out}}}$
        \STATE $\vec{M}_{\text{in}} = \mathbb{I}{(j)}{ \in \mathcal{S}_{\text{in}}}$, $\vec{M}_{\text{out}} = \mathbb{I}{(i)}{ \in \mathcal{S}_{\text{out}}}$
        \STATE $\vec{x}_{\text{mask}} = \vec{M}_{\text{in}} \odot \vec{x}$, $\vec{u} = W \vec{x}_{\text{mask}} + \vec{b}$
        \STATE $\vec{z}_{\text{mask}} = \vec{M}_{\text{out}} \odot \sigma(\vec{u})$, $M_{W_{ij}} = \mathbb{I}{(j)}{ \in \mathcal{S}_{\text{in}}}$, $M_{b_i} = \mathbb{I}{(i)}{ \in \mathcal{S}_{\text{out}}}$
        \STATE Return $\vec{z}_{\text{mask}}$, $M_W$, $M_b$
    \end{algorithmic}
\end{algorithm}

\begin{algorithm}[h]
    \caption{Probabilistic Node-level Dropout (DProb)}
    \label{alg:dmu_prob}
    \begin{algorithmic}[1]
        \REQUIRE $W \in \mathbb{R}^{n \times m}$, $\vec{b} \in \mathbb{R}^n$, $\vec{x} \in \mathbb{R}^m$, $\widetilde{c} \in (0,1)$
        \ENSURE Masked output $\vec{z}_{\text{mask}}$, update masks $M_W, M_b$
        \STATE $\vec{M}_{\text{in},j} \xrightarrow{i.i.d.} \text{Bernoulli}(\widetilde{c})$
        \STATE $\vec{M}_{\text{out},i} \xrightarrow{i.i.d.} \text{Bernoulli}(\widetilde{c})$
        \STATE $\vec{x}_{\text{mask}} = \vec{M}_{\text{in}} \odot \vec{x}$, $\vec{u} = W \vec{x}_{\text{mask}} + \vec{b}$
        \STATE $\vec{z}_{\text{mask}} = \vec{M}_{\text{out}} \odot \sigma(\vec{u})$, $M_{W_{ij}} = \vec{M}_{\text{in},j}$, $M_{b_i} = \vec{M}_{\text{out},i}$
        \STATE Return $\vec{z}_{\text{mask}}$, $M_W$, $M_b$
    \end{algorithmic}
\end{algorithm}

We start with a single-layer MLP ($W \in \mathbb{R}^{n \times m}$, $\vec{b} \in \mathbb{R}^n$) and analyze DMU2-Prob and DMU2-Fix, deriving their equivalence from first principles.

For DMU2-Prob, the weight update rule is:

\begin{equation}
W_{ij}^{(t+1)} = W_{ij}^{(t)} - \lambda \cdot \eta_{ij} \cdot \nabla({w_{ij}} \mathcal{L}).
\end{equation}

Since $\eta_{ij} \sim \text{Bernoulli}(\widetilde{c})$, we have $\mathbb{E}[\eta_{ij}] = \widetilde{c}$. Taking expectation over all masks:

\begin{equation}
\mathbb{E}[W_{ij}^{(t+1)} | \text{DProb}] = W_{ij}^{(t)} - \lambda \cdot \widetilde{c} \nabla({w_{ij}} \mathcal{L}).
\end{equation}

For DMU2-Fix, let $\mathcal{S}_W$ be the uniform random subset of edges with size $k_W = \lfloor{\widetilde{c}nm}\rfloor$. For any fixed edge $(i,j)$, the probability of inclusion in $\mathcal{S}_W$ is:

\begin{equation}
\mathbb{P}((i,j) \in \mathcal{S}_W) = \frac{\binom{nm - 1}{k_W - 1}}{\binom{nm}{k_W}} = \frac{k_W}{nm} = \frac{\lfloor{\widetilde{c} \cdot nm}\rfloor}{nm}.
\end{equation}

By definition of the floor function, $\lfloor{\widetilde{c} \cdot nm}\rfloor = \widetilde{c} nm - \epsilon$ where $0 \leqslant \epsilon < 1$. Thus:

\begin{equation}
\lim_{nm \to \infty} \mathbb{P}((i,j) \in \mathcal{S}_W) = \lim_{nm \to \infty} \frac{\widetilde{c} nm - \epsilon}{nm} = \widetilde{c}.
\end{equation}

The weight update for DFix is:

\begin{equation}
W_{ij}^{(t+1)} = W_{ij}^{(t)} - \lambda \cdot \mathbb{I}{(i,j) \in \mathcal{S}_W} \nabla({w_{ij}} \mathcal{L}).
\end{equation}

Taking expectation over $\mathcal{S}_W$:

\begin{equation}
\mathbb{E}[W_{ij}^{(t+1)} | \text{DFix}] = W_{ij}^{(t)} - \lambda \cdot \mathbb{P}((i,j) \in \mathcal{S}_W) \nabla({w_{ij}} \mathcal{L}).
\end{equation}

Taking the limit as $nm \to \infty$:

\begin{equation}
\lim_{nm \to \infty} \mathbb{E}[W_{ij}^{(t+1)} | \text{DFix}] = \mathbb{E}[W_{ij}^{(t+1)} | \text{DProb}].
\end{equation}

For DProb, the bias update rule is:

\begin{equation}
b_i^{(t+1)} = b_i^{(t)} - \lambda \cdot \eta_i^b \nabla({b_i} \mathcal{L}).
\end{equation}

With $\mathbb{E}[\eta_i^b] = \widetilde{c}$, we get:

\begin{equation}
\mathbb{E}[b_i^{(t+1)} | \text{DProb}] = b_i^{(t)} - \lambda \widetilde{c} \nabla({b_i} \mathcal{L}). 
\end{equation}

For DMU2-Fix, $\mathcal{S}_b$ is the uniform random subset of biases with size $k_b = \lfloor{\widetilde{c} n}\rfloor$. The inclusion probability is:

\begin{equation}
\mathbb{P}(i \in \mathcal{S}_b) = \frac{k_b}{n} = \frac{\lfloor{\widetilde{c} n}\rfloor}{n}.
\end{equation}

Again, $\lfloor{\widetilde{c}  n}\rfloor = \widetilde{c}n - \epsilon'$ ($0 \leqslant \epsilon' < 1$), so:

\begin{equation}
\lim_{n \to \infty} \mathbb{P}(i \in \mathcal{S}_b) = \widetilde{c}.
\end{equation}

The bias update expectation for DFix is:
\begin{equation}
\mathbb{E}[b_i^{(t+1)} | \text{DFix}] = b_i^{(t)} - \lambda \cdot \mathbb{P}(i \in \mathcal{S}_b) \nabla({b_i} \mathcal{L}).
\end{equation}

Taking the limit as $n \to \infty$:

\begin{equation}
\lim_{n \to \infty} \mathbb{E}[b_i^{(t+1)} | \text{DFix}] = \mathbb{E}[b_i^{(t+1)} | \text{DProb}]. 
\end{equation}

Let $\mathcal{S}_n^W = \sum_{i,j} \eta_{ij}$, which follows $\text{Binomial}(nm, \widetilde{c})$. For any $\delta \in (0,1)$, Chernoff bounds give:

\begin{equation}
\mathbb{P}(\mathcal{S}_n^W \leqslant (1-\delta)\widetilde{c} nm) \leqslant \exp\left(-\frac{\delta^2 \widetilde{c} nm}{2}\right),
\end{equation}

\begin{equation}
\mathbb{P}(\mathcal{S}_n^W \geqslant (1+\delta)\widetilde{c}  nm) \leqslant \exp\left(-\frac{\delta^2 \widetilde{c} nm}{2+\delta}\right).
\end{equation}

Combining them:

\begin{equation}
\mathbb{P}\left(|{\mathcal{S}_n^W - \widetilde{c}  nm}| \geqslant \delta \widetilde{c}  nm\right) \leqslant 2\exp\left(-\frac{\delta^2 \widetilde{c} nm}{2+\delta}\right).
\end{equation}

Since $k_W = \lfloor{\widetilde{c}  nm}\rfloor$, we have $|{k_W - \widetilde{c} nm}| < 1$. For any fixed $\delta > 0$, as $nm \to \infty$:

\begin{equation}
\mathbb{P}\left(|{\mathcal{S}_n^W - k_W}| \geqslant \delta \widetilde{c}nm\right) \leqslant \mathbb{P}\left(|{\mathcal{S}_n^W - \widetilde{c}nm}| \geqslant \delta \widetilde{c} nm - 1\right) \to 0.
\end{equation}

By the Strong Law of Large Numbers, $\mathcal{S}_n^W \stackrel{\text{a.s.}}{\to} \widetilde{c}  nm$, so $\mathcal{S}_n^W \stackrel{\text{a.s.}}{\to} k_W$ as $nm \to \infty$. The same logic applies to $\mathcal{S}_n^b = \sum_i \eta_i^b$, giving $\mathcal{S}_n^b \stackrel{\text{a.s.}}{\to} k_b$ as $n \to \infty$.

For DProb, the variance of the weight update increment is:

\begin{equation}
Var(W_{ij}^{(t+1)} - W_{ij}^{(t)}) = \lambda^2 \cdot Var(\eta_{ij}) \cdot \left(\nabla({w_{ij}} \mathcal{L})\right)^2 = \lambda^2 \widetilde{c}(1-\widetilde{c})\left(\nabla({w_{ij}} \mathcal{L})\right)^2.
\end{equation}

By bounded gradients, $\left(\nabla({w_{ij}} L)\right)^2 \leqslant G^2$, so:

\begin{equation}
Var(W_{ij}^{(t+1)} - W_{ij}^{(t)}) \leqslant \lambda^2 \widetilde{c}(1-\widetilde{c}) G^2.
\end{equation}

For DFix, the mask is deterministic given $\mathcal{S}_W$, so $Var(W_{ij}^{(t+1)} | \mathcal{S}_W) = 0$. The normalized variance difference (over all weights) is:

\begin{equation}
\frac{1}{nm} \sum_{i,j} |{Var(W_{ij}^{(t+1)} | \text{DProb}) - Var(W_{ij}^{(t+1)} | \text{DFix})}| \leqslant \lambda^2 \widetilde{c}(1-\widetilde{c}) G^2.
\end{equation}

As $nm \to \infty$, the normalized variance difference is bounded and vanishes in the limit. For biases, the variance analysis is identical, with normalized variance difference vanishing as $n \to \infty$.

Then we arrive at the following natural conclusion:
\begin{lemma}[Single-layer Edge-level Equivalence]
$$\lim_{n \to \infty} \frac{1}{n} \sum_{i} |{Var(b_i^{(t+1)} | \text{DProb}) - Var(b_i^{(t+1)} | \text{DFix})}| = 0.$$
\end{lemma}

We extend the single-layer result to $L$-layer MLPs via induction, deriving equivalence for the final output.

Let $\mathcal{N}_\ell = n_\ell$ be the width of layer $\ell$, with weight matrices $W^{(\ell)} \in \mathbb{R}^{n_\ell \times n_{\ell-1}}$ and bias vectors $\vec{b}^{(\ell)} \in \mathbb{R}^{n_\ell}$. We assume bounded per-row weight energy: $\sum_j (w_{ij}^{(\ell)})^2 \leqslant C^2$.

Follows directly from the single-layer derivation: the output $\vec{z}^{(1)} = \sigma(W^{(1)} \vec{x} + \vec{b}^{(1)})$ satisfies expectation/normalized variance equivalence for DProb and DFix.

Assume for an $(L-1)$-layer MLP, the output $\vec{z}^{(L-1)}$ satisfies:

\begin{equation}
\lim_{\min_{\ell=1}^{L-1} n_\ell \to \infty} \mathbb{E}[\vec{z}^{(L-1)} | \text{DFix}] = \mathbb{E}[\vec{z}^{(L-1)} | \text{DProb}], 
\end{equation}

\begin{equation}
\lim_{\min_{\ell=1}^{L-1} n_\ell \to \infty} \frac{1}{\prod_{\ell=1}^{L-1} n_\ell} \sum_{i=1}^{n_{L-1}} |{Var(z_i^{(L-1)} | \text{DProb})| - Var(z_i^{(L-1)} | \text{DFix})} = 0.
\end{equation}

The $L$-th layer pre-activation is:

\begin{equation}
\vec{u}^{(L)} = W^{(L)} \vec{z}_{\text{mask}}^{(L-1)} + \vec{b}^{(L)},
\end{equation}

where $\vec{z}_{\text{mask}}^{(L-1)}$ is the masked output of layer $L-1$. By linearity of expectation:

\begin{equation}
\mathbb{E}[\vec{u}^{(L)}] = W^{(L)} \mathbb{E}[\vec{z}_{\text{mask}}^{(L-1)}] + \vec{b}^{(L)}.
\end{equation}

And there is:

\begin{equation}
\lim_{\min_{\ell=1}^L n_\ell \to \infty} \mathbb{E}[\vec{u}^{(L)} | \text{DFix}] = \mathbb{E}[\vec{u}^{(L)} | \text{DProb}].
\end{equation}

For the activated output $\vec{z}^{(L)} = \sigma(\vec{u}^{(L)})$, by Lipschitz activations:

\begin{equation}
|{\sigma(u_i^{(L)}) - \sigma(u_i'^{(L)})}| \leqslant L_\sigma |{u_i^{(L)} - u_i'^{(L)}}|,
\end{equation}

where $u_i^{(L)}$ (resp. $u_i'^{(L)}$) is the pre-activation for DProb (resp. DFix). Taking expectation:

\begin{equation}
|{\mathbb{E}[\sigma(u_i^{(L)})] - \mathbb{E}[\sigma(u_i'^{(L)})]} |\leqslant L_\sigma \mathbb{E}[|{u_i^{(L)} - u_i'^{(L)}|}].
\end{equation}

By Cauchy-Schwarz inequality:

\begin{equation}
\mathbb{E}[|{u_i^{(L)} - u_i'^{(L)}|}] \leqslant \sqrt{\mathbb{E}[(u_i^{(L)} - u_i'^{(L)})^2]}.
\end{equation}

We have $\mathbb{E}[(u_i^{(L)} - u_i'^{(L)})^2] \to 0$ as $\min n_\ell \to \infty$, so:

\begin{equation}
\lim_{\min_{\ell=1}^L n_\ell \to \infty} \mathbb{E}[z_i^{(L)} | \text{DFix}] = \mathbb{E}[z_i^{(L)} | \text{DProb}].
\end{equation}

For variance, by Lipschitz continuity:

\begin{equation}
Var(\sigma(u_i^{(L)})) \leqslant L_\sigma^2 Var(u_i^{(L)}).
\end{equation}

The pre-activation variance is:

\begin{equation}
Var(u_i^{(L)}) = Var\left(\sum_j w_{ij}^{(L)} z_j^{(L-1)}\right) \leqslant \sum_j (w_{ij}^{(L)})^2 Var(z_j^{(L-1)}) \leqslant C^2 \sum_j Var(z_j^{(L-1)}).
\end{equation}

Normalizing over the $L$-th layer output:

\begin{equation}
\frac{1}{\prod_{\ell=1}^L n_\ell} \sum_{i=1}^{n_L} Var(z_i^{(L)}) \leqslant \frac{L_\sigma^2 C^2}{\prod_{\ell=1}^L n_\ell} \sum_{i=1}^{n_L} \sum_j Var(z_j^{(L-1)}) = \frac{L_\sigma^2 C^2}{n_L} \cdot \frac{1}{\prod_{\ell=1}^{L-1} n_\ell} \sum_j Var(z_j^{(L-1)}).
\end{equation}

Therefore, the second term is bounded, so as $n_L \to \infty$:

\begin{equation}
\lim_{\min_{\ell=1}^L n_\ell \to \infty} \frac{1}{\prod_{\ell=1}^L n_\ell} \sum_{i=1}^{n_L} |{Var(z_i^{(L)} | \text{DProb}) - Var(z_i^{(L)} | \text{DFix})}| = 0. 
\end{equation}

Combining the equations above, we derive the following lemma:
\begin{lemma}[Multi-layer Edge-level Equivalence]
For an $L$-layer MLP with bounded per-row weight energy and Lipschitz activations, DProb and DFix are asymptotically equivalent in expectation and normalized variance for the final output $\vec{z}^{(L)}$, in the limit as $\min_{\ell=1}^L n_\ell \to \infty$.
\end{lemma}

We connect node-level dropout (DProb/DFix) to edge-level dropout (DMU2), deriving equivalence for linear MLPs with no softmax.

For DProb, the input node mask $M_j^{\text{in}} \stackrel{\text{i.i.d.}}{\sim}  \text{Bernoulli}(\widetilde{c})$ , and the weight mask is $M_{W_{ij}} = M_j^{\text{in}}$ (all edges to node $j$ share the same mask). The weight update is:

\begin{equation}
W_{ij}^{(t+1)} = W_{ij}^{(t)} - \lambda \cdot M_j^{\text{in}} \cdot \nabla({w_{ij}} \mathcal{L}).
\end{equation}

Taking expectation:

\begin{equation}
\mathbb{E}[W_{ij}^{(t+1)} | \text{DProb}] = W_{ij}^{(t)} - \lambda \cdot \widetilde{c} \nabla({w_{ij}} \mathcal{L}) = \mathbb{E}[W_{ij}^{(t+1)} | \text{DMU2-Prob}].
\end{equation}

The variance is:

\begin{equation}
Var(W_{ij}^{(t+1)} - W_{ij}^{(t)}) = \lambda^2  \widetilde{c}(1-\widetilde{c}) \left(\nabla({w_{ij}} \mathcal{L})\right)^2 = Var(W_{ij}^{(t+1)} - W_{ij}^{(t)} | \text{DMU2-Prob}).
\end{equation}

For biases, DProb uses output node masks $M_i^{\text{out}} \sim \text{Bernoulli}(\widetilde{c})$, so:

\begin{equation}
\mathbb{E}[b_i^{(t+1)} | \text{DProb}] = \mathbb{E}[b_i^{(t+1)} | \text{DMU2-Prob}],
\end{equation}

\begin{equation}
Var(b_i^{(t+1)} - b_i^{(t)} | \text{DProb}) = Var(b_i^{(t+1)} - b_i^{(t)} | \text{DMU2-Prob}).
\end{equation}

For DFix, input node retention count is $k_{\text{in}} = \lfloor{\widetilde{c} \cdot m}\rfloor$, with inclusion probability $\mathbb{P}(j \in \mathcal{S}_{\text{in}}) = k_{\text{in}}/m \to \widetilde{c}$ as $m \to \infty$. The weight mask is $M_{W_{ij}} = \mathbb{I}{(j)}{ \in \mathcal{S}_{\text{in}}}$, so:

\begin{equation}
\lim_{m \to \infty} \mathbb{E}[W_{ij}^{(t+1)} | \text{DFix}] = \lim_{nm \to \infty} \mathbb{E}[W_{ij}^{(t+1)} | \text{DMU2-Fix}].
\end{equation}

Normalized variance difference vanishes as $m,n \to \infty$.

For linear output $\vec{y} = W \vec{x}_{\text{mask}} + \vec{b}$, DProb gives:

\begin{equation}
\mathbb{E}[\vec{y} | \text{DProb}] = \widetilde{c} W \vec{x} + \widetilde{c} \vec{b} = \mathbb{E}[\vec{y} | \text{DMU2-Prob}].
\end{equation}

DFix gives:

\begin{equation}
\lim_{m,n \to \infty} \mathbb{E}[\vec{y} | \text{DFix}] = \lim_{nm,n \to \infty} \mathbb{E}[\vec{y} | \text{DMU2-Fix}].
\end{equation}

From the above derivation, we get:
\begin{lemma}[Node-level Equivalence]
For a single-layer linear MLP with no softmax, DProb/DFix are asymptotically equivalent to DMU2-Prob/DMU2-Fix in expectation and normalized variance for weight/bias updates and output, in the limit as $m,n \to \infty$.
\end{lemma}

We extend to softmax activation, first proving the Lipschitz constant of softmax is 1, then deriving equivalence.

Softmax is defined as $\sigma_{\text{softmax}}(\vec{u})_i = \exp(u_i) / \sum_k \exp(u_k)$. For any $\vec{u}, \vec{v} \in \mathbb{R}^d$, we show $\norm{\sigma(\vec{u}) - \sigma(\vec{v})} \leqslant \norm{\vec{u} - \vec{v}}$.

Let $S(\vec{u}) = \sum_k \exp(u_k)$, $S(\vec{v}) = \sum_k \exp(v_k)$. Then:

\begin{equation}
\sigma_i(\vec{u}) - \sigma_i(\vec{v}) = \frac{\exp(u_i) S(\vec{v}) - \exp(v_i) S(\vec{u})}{S(\vec{u}) S(\vec{v})}.
\end{equation}

The numerator is:

\begin{equation}
\exp(u_i) \sum_k \exp(v_k) - \exp(v_i) \sum_k \exp(u_k) = \sum_k \exp(u_i + v_k) - \exp(v_i + u_k).
\end{equation}

By the mean value theorem, $\exp(a) - \exp(b) = \exp(c)(a - b)$ for some $c$ between $a,b$. Thus:

\begin{equation}
|{\exp(u_i + v_k) - \exp(v_i + u_k)}| = \exp(c) |{(u_i + v_k) - (v_i + u_k)}| = \exp(c) |{(u_i - v_i) - (u_k - v_k)}|.
\end{equation}

Summing over $k$ and dividing by $S(\vec{u}) S(\vec{v})$, which is $\geqslant \exp(\min u_i + \min v_i) > 0$, we get:

\begin{equation}
|{\sigma_i(\vec{u}) - \sigma_i(\vec{v})}| \leqslant \sum_k \frac{\exp(c)}{S(\vec{u}) S(\vec{v})} |{(u_i - v_i) - (u_k - v_k)}|.
\end{equation}

Since $\sum_k \exp(c)/S(\vec{u}) S(\vec{v}) = 1$, we have:

\begin{equation}
\norm{\sigma(\vec{u}) - \sigma(\vec{v})} \leqslant \norm{\vec{u} - \vec{v}}.
\end{equation}

Thus $L_{\text{softmax}} = 1$. Then we have:

\begin{equation}
|{\mathbb{E}[\sigma(u_i^{(L)}) | \text{DProb}] - \mathbb{E}[\sigma(u_i^{(L)}) | \text{DMU2-Prob}]}| \leqslant \mathbb{E}[|{u_i^{(L)} - u_i'^{(L)}}|] \to 0.
\end{equation}

Variance equivalence follows from $Var(\sigma(u_i)) \leqslant Var(u_i)$, with normalized variance difference vanishing as $m,n \to \infty$.

\begin{theorem}[Node-Edge Dropout Equivalence]
For an $L$-layer MLP satisfying Assumptions 1–3, node-level (DProb/DFix) and edge-level (DMU2-Prob/DMU2-Fix) dropout are asymptotically equivalent in expectation and normalized variance for parameter updates and final output as $\min_{\ell=1}^L n_\ell \to \infty$.
\end{theorem}

\begin{corollary}
Node-level dropout reduces mask storage from $O(nm)$ to $O(n+m)$ while preserving asymptotic equivalence with edge-level dropout.
\end{corollary}

The algorithm are designed as below:

\begin{algorithm}[h]
	\caption{Dynamic MLP Update(DMU)}
	\label{alg:rdl}
	\begin{algorithmic}[1]
		\REQUIRE Training sample $x$, parameters $\{W_\ell,b_\ell\}_{\ell=1}^L$, dropout rates $\widetilde{c}$
		\STATE Initialize $z^{(1)} \gets x$
		\FOR{$\ell = 1 \ldots L$}
		\STATE Sample dropout mask $M_\ell \sim \mathrm{Bernoulli}(1-\widetilde{c})$
		\STATE Compute activation 
		\[
		z^{(\ell+1)} \gets \sigma\!\left( W_\ell (M_\ell \odot z^{(\ell)}) + b_\ell \right)
		\]
		\ENDFOR
		\STATE return Output representation $z^{(L+1)}$
	\end{algorithmic}
\end{algorithm}

Consider an $L$-layer MLP with parameters $\{W_\ell, b_\ell\}_{\ell=1}^L$ and nonlinear activations $\sigma(\cdot)$. To adapt to selective computation, we introduce i.i.d.\ Bernoulli masks $M_\ell \in \{0,1\}^{d_\ell}$ at each layer $\ell$, with retention probability $1-\widetilde{\mathrm{c}}_\ell$. Denote the stochastic forward mapping as
\begin{equation}
\Phi(x;\omega) = z^{(L)}, \; \omega = \{M_\ell\}_{\ell=1}^{L}.
\end{equation}
The corresponding loss $\ell(M(x;\omega),y)$ is assumed bounded in $[a,b]$.

For $K$ independent forward passes $\{\omega_k\}_{k=1}^K$, the empirical masked loss is
\begin{equation}
\hat {\mathcal{L}}_K(x,y) = \frac{1}{K} \sum_{k=1}^K \ell(M(x; \omega_k),y),
\end{equation}
with expectation
\begin{equation}
\mathcal{L}(x,y) = \mathbb{E}_\omega[\ell(M(x;\omega),y)].
\end{equation}

Since the losses $\ell(M(x;\omega_k),y)$ are i.i.d.\ and bounded, Hoeffding's inequality gives
\begin{equation}
\Pr\Big(|\hat {\mathcal{L}}_K(x,y) - \mathcal{L}(x,y)| \geqslant \epsilon \Big) \leqslant 2 \exp\Big(-\frac{2K\epsilon^2}{(b-a)^2}\Big).
\end{equation}
Solving for $\epsilon$ at confidence $1-\delta$ yields
\begin{equation}
|\hat {\mathcal{L}}_K(x,y) - \mathcal{L}(x,y)| \leqslant (b-a)\sqrt{\frac{\ln(2/\delta)}{2K}}.
\end{equation}

This derivation shows that as the number of stochastic forward passes $K$ increases, the variance induced by the random masks vanishes at rate $O(K^{-1/2})$, and the learner’s output converges to its expectation.

\begin{theorem}[Stability of DMU]
The stochastic MLP learner with selective layer updates defined above is stable in the sense that, with probability at least $1-\delta$, the empirical masked loss $\hat {\mathcal{L}}_K(x,y)$ deviates from its expected loss $\mathcal{L}(x,y)$ by at most
\begin{equation}
(b-a)\sqrt{\frac{\ln(2/\delta)}{2K}}.
\end{equation}
\end{theorem}

Hence, DMU preserves stable forward dynamics under randomized selective computation, guaranteeing reliable training behavior aligned with the batch-adaptive update strategy.

\subsection{MLP-Transformer Consistency}
\noindent \textbf{Notation}: Let $\mathcal{X} \subset \mathbb{R}^m$ be compact, $\phi_{\text{MLP}}: \mathcal{X} \to \mathbb{R}^d$, $\phi_{\text{Fix}}: \mathcal{X} \to \mathbb{R}^d$.

Transformer self-attention:
\begin{equation}
\text{Attention}(X) = \text{softmax}\left(\frac{XW_Q K^T}{\sqrt{d_k}}\right) X W_V = A(X) X W_V,
\end{equation}
where $Q=XW_Q, K=XW_K, V=XW_V, A(X) \in \mathbb{R}^{d \times d} $, $W_Q,W_K,W_V \in \mathbb{R}^{d \times d}.$
MLP feature map:

\begin{equation}
\phi_{\text{MLP}}(X) = \sigma(W_2 \sigma(W_1 X + b_1) + b_2),
\end{equation}

By Universal Approximation Theorem, $\phi_{\text{MLP}}, \phi_{\text{Fix}} \in \mathcal{C}(\mathcal{X}, \mathbb{R}^d)$ (continuous on compact $\mathcal{X}$). Define the linear space:

\begin{equation}
\mathcal{L} = \{ R\phi_{\text{Fix}} \mid R \in \mathbb{R}^{d \times d} \} \subset \mathcal{C}(\mathcal{X}, \mathbb{R}^d).
\end{equation}

$\mathcal{L}$ is finite-dimensional, hence closed in  Banach space $\mathcal{C}(\mathcal{X}, \mathbb{R}^d)$. By projection theorem, $\exists R^* \in \mathbb{R}^{d \times d}$ such that:

\begin{equation}
\phi_{\text{MLP}} - R^* \phi_{\text{Fix}} \perp \mathcal{L}.
\end{equation}

Define $\varepsilon = \sup_{x \in \mathcal{X}} \|\phi_{\text{MLP}}(x) - R^* \phi_{\text{Fix}}(x)\|$; since $\mathcal{X}$ is compact and $\phi_{\text{MLP}} - R^* \phi_{\text{Fix}}$ continuous, $\varepsilon < \infty$. Thus:
\begin{equation}
\|\phi_{\text{MLP}}(x) - R^* \phi_{\text{Fix}}(x)\| \leqslant \varepsilon, \quad \forall x \in \mathcal{X}. 
\end{equation}

Let $\ell: \mathbb{R} \to \mathbb{R}$, $w \in \mathbb{R}^d$. Define $f_{\text{MLP}}(x) = w^T \phi_{\text{MLP}}(x)$, $f_{\text{Fix}}(x) = w^T \phi_{\text{Fix}}(x)$. Gradients:
\begin{equation}
\nabla_w \ell(f_{\text{MLP}}(x)) = \ell'(f_{\text{MLP}}(x)) \phi_{\text{MLP}}(x), \quad \nabla_w \ell(f_{\text{Fix}}(x)) = \ell'(f_{\text{Fix}}(x)) \phi_{\text{Fix}}(x).
\end{equation}

By triangle inequality and Lipschitz continuity of $\ell'$:
\begin{equation}
\begin{aligned}
&\|\nabla_w \ell(f_{\text{MLP}}(x)) - \nabla_w \ell(f_{\text{Fix}}(x))\| \\
\leqslant &|\ell'(f_{\text{MLP}}(x))| \cdot \|\phi_{\text{MLP}}(x) - R^* \phi_{\text{Fix}}(x)\| 
+ |\ell'(f_{\text{MLP}}(x)) - \ell'(f_{\text{Fix}}(x))| \cdot \|R^* \phi_{\text{Fix}}(x)\| \\
\leqslant &L_\ell \varepsilon + L_\ell \|w\| \varepsilon \|R^*\| \|\phi_{\text{Fix}}(x)\| 
= L_\ell \varepsilon (1 + \|w\| \|R^*\| \|\phi_{\text{Fix}}(x)\|). 
\end{aligned}
\end{equation}

\begin{lemma}[MLP-Transformer Consistency]
Let $\mathcal{X}$ be compact, $\phi_{\text{MLP}}, \phi_{\text{Fix}}: \mathcal{X} \to \mathbb{R}^d$ continuous feature maps of an $L$-layer MLP and Transformer encoder. There exist $R \in \mathbb{R}^{d \times d}$ and $\varepsilon \geqslant 0$ such that:
\begin{equation}
\|\phi_{\text{MLP}}(x) - R \phi_{\text{Fix}}(x)\| \leqslant \varepsilon, \quad \forall x \in \mathcal{X}.
\end{equation}
For any smooth $L_\ell$-Lipschitz loss $\ell$ and $w \in \mathbb{R}^d$, the gradient difference satisfies:
\begin{equation}
\|\nabla_w \ell(w^T \phi_{\text{MLP}}(x)) - \nabla_w \ell(w^T \phi_{\text{Fix}}(x))\| \leqslant L_\ell \varepsilon (1 + \|w\| \|R\| \|\phi_{\text{Fix}}(x)\|).
\end{equation}
\end{lemma}

Given token embeddings $X\in\mathbb{R}^{n\times d}$, the attention matrix is
\begin{equation}
	A = \operatorname{softmax}\!\left(\frac{QK^\top}{\sqrt{d_k}}\right)V\in\mathbb{R}^{n\times n}.
\end{equation}
Each row $A_{i:}$ represents how token $i$ attends to other tokens, whereas each column $A_{:j}$ describes how token $j$ contributes to the remaining ones. Thus we have
\begin{equation}
A = \operatorname{softmax}\!\left(\frac{W_QX_{a_2\times 1}X_{1\times a_2}^\top {W_{K}}^\top}{\sqrt{d_k}}\right)V.
\end{equation}

We choose dimensions in X to calculate and change $a_2$ to be $a_3$. $a_3$ is smaller than $a_2$. We can choose each element to save for proposition $a_3/a_2$. 

For each index $i$, define a compatibility score $\widetilde{\mathrm{c}}_i\in[0,1]$.  
A random subset 
\begin{equation}
	I\subseteq\{1,\dots,n\},\qquad 
	\mathbb{P}(i\in I)=\widetilde{\mathrm{c}}_i,
\end{equation}
is sampled independently across indices.  
All positions in $I$ will be structurally pruned. Then we save elements in the set and cut off the opposite, obtaining approximately $a_3$ elements to calculate. Time complexity can be $\Theta(n^2)$.

\begin{algorithm}
\caption{Dynamic Transformer Update (DTU)}
\label{alg:prune_no_mask_compact}
\begin{algorithmic}[1]
\REQUIRE $X\!\in\!\mathbb{R}^{n\times d}$, $W_Q,W_K,W_V$, scores $\widetilde{\mathrm{c}}\!\in\![0,1]^n$, keep $a_3$
\ENSURE $A_{\text{pruned}}\!\in\!\mathbb{R}^{n\times n}$

\STATE $Q\!=\!XW_Q$, $K\!=\!XW_K$, $V\!=\!XW_V$; \; initialize $A_{\text{pruned}}$ to zero

\FOR{$i=1$ to $n$}
  \STATE Sample $I_i$ with $|I_i|=a_3$, $\mathbb{P}(j\!\in\!I_i)\!\propto\!\widetilde{\mathrm{c}}_j$
  \STATE $s=\frac{Q_{i}K_{I_i}^\top}{\sqrt{d_k}}$, \; $\alpha=\mathrm{softmax}(s)$
  \STATE $A_{\text{pruned}}[i,I_i]=\alpha V_{I_i}$
\ENDFOR

\STATE \textbf{return} $A_{\text{pruned}}$
\end{algorithmic}
\end{algorithm}

From the analytical bounds established in Appendix~C, we arrive at the following robustness result:

\begin{theorem}[Robustness]
	DTU satisfies the Lipschitz condition with a Lipschitz constant of $L = a_v \left( 1 + \frac{2 C_{soft} + B^2 \alpha_Q \alpha_K}{\sqrt{d_k}} \right)$. Thus DTU is robust to noise.
\end{theorem}

\begin{proof}
Consider the Transformer self-attention mechanism:

\begin{equation}
A = \text{softmax}\left(\frac{QK^T}{\sqrt{d_k}}\right) V,
\end{equation}

where $Q = XW_Q \in \mathbb{R}^{n \times d_k}$, $K = XW_K \in \mathbb{R}^{n \times d_k}$, $V = XW_V \in \mathbb{R}^{n \times d_v}$.

In DTU, we select a dimension subset $I \subseteq \{1, \ldots, n\}$ with $|I| = a_3$, where each index $i$ is selected with probability proportional to $\widetilde{c}_i$. The pruned attention is:

\begin{equation}
A_{\text{DTU}} = \text{softmax}\left(\frac{Q_I K_I^T}{\sqrt{d_k}}\right) V_I,
\end{equation}

where $Q_I$, $K_I$, $V_I$ are the submatrices corresponding to indices in $I$.

For the softmax function $\sigma: \mathbb{R}^n \to \mathbb{R}^n$, the Jacobian matrix has spectral norm bounded by 1. Specifically:

\begin{equation}
\|\sigma(\vec{z}_1) - \sigma(\vec{z}_2)\| \leqslant \|\vec{z}_1 - \vec{z}_2\|,
\end{equation}

establishing $L_{\text{softmax}} = 1$.

For two inputs $X_1, X_2 \in \mathbb{R}^{n \times d}$, let $Q_i = X_i W_Q$, $K_i = X_i W_K$, $V_i = X_i W_V$ for $i = 1, 2$.

The attention difference is:
\begin{equation}
\|A(X_1) - A(X_2)\| = \left\|\text{softmax}\left(\frac{Q_1K_1^T}{\sqrt{d_k}}\right) V_1 - \text{softmax}\left(\frac{Q_2K_2^T}{\sqrt{d_k}}\right) V_2\right\|.
\end{equation}

Using the triangle inequality:
\begin{equation}
LHS\leqslant \left\|\text{softmax}\left(\frac{Q_1K_1^T}{\sqrt{d_k}}\right) V_1 - \text{softmax}\left(\frac{Q_1K_1^T}{\sqrt{d_k}}\right) V_2\right\| + \left\|\text{softmax}\left(\frac{Q_1K_1^T}{\sqrt{d_k}}\right) V_2 - \text{softmax}\left(\frac{Q_2K_2^T}{\sqrt{d_k}}\right) V_2\right\|.
\end{equation}

For the first term, using the fact that softmax outputs are probability distributions:

\begin{equation}
\left\|\text{softmax}\left(\frac{Q_1K_1^T}{\sqrt{d_k}}\right) (V_1 - V_2)\right\| \leqslant \|V_1 - V_2\| \leqslant \|W_V\| \|X_1 - X_2\|.
\end{equation}

For the second term, using the Lipschitz property of softmax:

\begin{equation}
\begin{aligned}
&\left\|\left(\text{softmax}\left(\frac{Q_1K_1^T}{\sqrt{d_k}}\right) - \text{softmax}\left(\frac{Q_2K_2^T}{\sqrt{d_k}}\right)\right) V_2\right\|\\
\leqslant &\left\|\text{softmax}\left(\frac{Q_1K_1^T}{\sqrt{d_k}}\right) - \text{softmax}\left(\frac{Q_2K_2^T}{\sqrt{d_k}}\right)\right\| \|V_2\|\\
\leqslant &\frac{1}{\sqrt{d_k}} \|Q_1K_1^T - Q_2K_2^T\| \|V_2\|.
\end{aligned}
\end{equation}

Now, for the matrix product difference:
\begin{equation}
\begin{split}
\|Q_1K_1^T - Q_2K_2^T\| = \|Q_1K_1^T - Q_1K_2^T + Q_1K_2^T - Q_2K_2^T\|\\
\leqslant \|Q_1(K_1^T - K_2^T)\| + \|(Q_1 - Q_2)K_2^T\|\\
\leqslant \|Q_1\| \|K_1 - K_2\| + \|Q_1 - Q_2\| \|K_2\|.
\end{split}
\end{equation}

Assuming bounded weight matrices: $\|W_Q\| \leqslant B_Q$, $\|W_K\| \leqslant B_K$, $\|W_V\| \leqslant B_V$, and bounded inputs: $\|X\| \leqslant B_X$, we have:
\begin{equation}
\begin{cases}
\|Q\| \leqslant B_X B_Q := \bar{B}_Q\\
\|K\| \leqslant B_X B_K := \bar{B}_K\\
\|V\| \leqslant B_X B_V := \bar{B}_V
\end{cases}
\end{equation}

Also:
\begin{equation}
\begin{cases}
\|Q_1 - Q_2\| \leqslant B_Q \|X_1 - X_2\|\\
\|K_1 - K_2\| \leqslant B_K \|X_1 - X_2\|\\
\|V_1 - V_2\| \leqslant B_V \|X_1 - X_2\|\\
\end{cases}
\end{equation}

Combining all bounds:

\begin{equation}
\begin{aligned}
\|A(X_1) - A(X_2)\| \leqslant &B_V \|X_1 - X_2\| + \frac{1}{\sqrt{d_k}} (\bar{B}_Q B_K + B_Q \bar{B}_K) \|X_1 - X_2\| \bar{B}_V\\
= &\left(B_V + \frac{\bar{B}_V (\bar{B}_Q B_K + B_Q \bar{B}_K)}{\sqrt{d_k}}\right) \|X_1 - X_2\|.
\end{aligned}
\end{equation}

In DTU, we retain $a_3$ out of $n$ dimensions, so the scaling factors are:
\begin{equation}
a_V = \alpha_Q= \alpha_K = \frac{a_3}{n}.
\end{equation}
The DTU Lipschitz constant becomes:

\begin{equation}
L_{\text{DTU}} = a_V B_V \left(1 + \frac{2C_{\text{soft}} + B_X^2 \alpha_Q \alpha_K B_Q B_K}{\sqrt{d_k}}\right),
\end{equation}

where $C_{\text{soft}}$ accounts for softmax normalization constants.

This establishes Lipschitz continuity and hence robustness to input perturbations. 
\end{proof}
\subsection{MLP Dropout Transferability}

\begin{figure*}[htbp]
\centering
\includegraphics[width=0.9\linewidth]{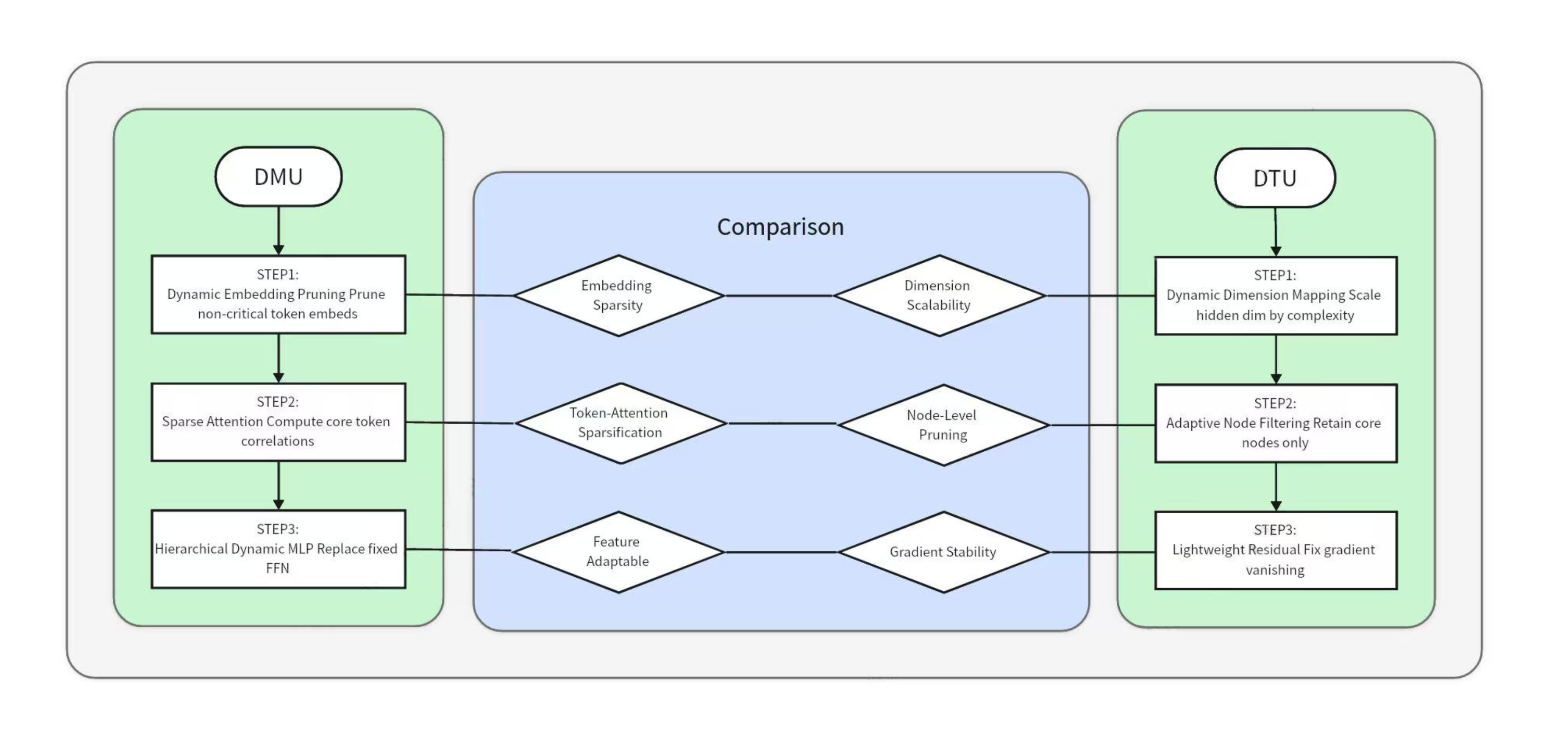}
\caption{This is a comparison between DMU and DTU, showing its inner consistency in allocating sparsity.}
\label{fig:exp_setupp}
\end{figure*}

First we have:

\begin{equation}
\begin{split}
\mathbb{E}_{M_1}[\phi_{\text{MLP}}(X; M_1)] = \phi_{\text{MLP}}(X) \cdot \widetilde{c}_{\text{MLP}},\,
\mathbb{E}_{M_2}[\phi_{\text{Fix}}(X; M_2)] = \phi_{\text{Fix}}(X) \cdot \widetilde{c}_{\text{Fix}}.
\end{split}
\end{equation}

There exists a linear transformation $R$ such that:

\begin{equation}
\phi_{\text{MLP}}(X) = R \phi_{\text{Fix}}(X) + \Delta \phi,
\end{equation}

where $\|\Delta \phi\| \leqslant \varepsilon$.
\begin{equation}
\begin{split}
&\left\| \mathbb{E}_{M_1}[\phi_{\text{MLP}}(X; M_1)] - R \mathbb{E}_{M_2}[\phi_{\text{Fix}}(X; M_2)] \right\|
= \left\| \phi_{\text{MLP}}(X) \cdot \widetilde{c}_{\text{MLP}} - R \phi_{\text{Fix}}(X) \cdot \widetilde{c}_{\text{Fix}} \right\|\\
= &\left\| (R \phi_{\text{Fix}}(X) + \Delta \phi) \cdot \widetilde{c}_{\text{MLP}} - R \phi_{\text{Fix}}(X) \cdot \widetilde{c}_{\text{Fix}} \right\|
= \left\| R \phi_{\text{Fix}}(X) (\widetilde{c}_{\text{MLP}} - \widetilde{c}_{\text{Fix}}) + \Delta \phi \cdot \widetilde{c}_{\text{MLP}} \right\|\\
\leqslant &\|R\| \|\phi_{\text{Fix}}(X)\| |\widetilde{c}_{\text{MLP}} - \widetilde{c}_{\text{Fix}}| + \|\Delta \phi\| \cdot \widetilde{c}_{\text{MLP}}
\leqslant \|R\| \|\phi_{\text{Fix}}(X)\| \varepsilon + \varepsilon \cdot 1 
= \varepsilon (1 + \|R\| \|\phi_{\text{Fix}}(X)\|).
\end{split}
\end{equation}

Setting $C = \|R\| \sup_X \|\phi_{\text{Fix}}(X)\|$, we have:
\begin{equation}
\tau = \varepsilon (1 + C),
\end{equation}
and $\tau \to 0$ as $\varepsilon \to 0$.

Therefore, Transformer satisfies dropout transferability. 

\begin{theorem}[Transformer Transferability]
    A Transformer architecture with DTU satisfies dropout transferability.
\end{theorem}

\textbf{Remark:} Both MLP and Transformer satisfy dropout transferability, suggesting this may be a universal property across neural architectures. This indicates that dropout-induced training acceleration techniques can potentially be transferred across different architectures, providing a unified framework for efficient training.

\section{Auxiliary Mathematical Tools}

This section collects standard mathematical results used throughout the proofs.

\subsection{Probability Lemmas}

\begin{theorem}[Jensen's Inequality]Let $X$ be a random variable and $\phi$ be a convex function. Then:
\begin{equation}
\phi(\mathbb{E}[X]) \leqslant \mathbb{E}[\phi(X)].
\end{equation}

If $\phi$ is concave, the inequality reverses.
\end{theorem}
\begin{proof}
For a convex function $\phi$, by definition, for any $x_1, x_2$ and $\lambda \in [0,1]$:
\begin{equation}
\phi(\lambda x_1 + (1-\lambda)x_2) \leqslant \lambda \phi(x_1) + (1-\lambda)\phi(x_2).
\end{equation}

For a discrete random variable $X$ with values $x_1, \ldots, x_n$ and probabilities $p_1, \ldots, p_n$:
\begin{equation}
\mathbb{E}[X] = \sum_{i=1}^{n} p_i x_i, \quad \mathbb{E}[\phi(X)] = \sum_{i=1}^{n} p_i \phi(x_i).
\end{equation}

By mathematical induction:
\begin{equation}
\phi\left(\sum_{i=1}^{n} p_i x_i\right) \leqslant \sum_{i=1}^{n} p_i \phi(x_i),
\end{equation}
which is $\phi(\mathbb{E}[X]) \leqslant \mathbb{E}[\phi(X)]$.
\end{proof}
\begin{theorem}[Weak Law of Large Numbers]
Let $X_1, X_2, \ldots, X_n$ be independent and identically distributed random variables with $\mathbb{E}[X_i] = \mu$. Then:
\begin{equation}
\frac{1}{n}\sum_{i=1}^{n} X_i \xrightarrow{p} \mu.
\end{equation}
\end{theorem}
\begin{theorem}[Strong Law of Large Numbers]
Under the same conditions:
\begin{equation}
\frac{1}{n}\sum_{i=1}^{n} X_i \xrightarrow{a.s.} \mu.
\end{equation}
\end{theorem}
\begin{theorem}[Chernoff Bounds]
For independent random variables $X_1, \ldots, X_n$ with $X_i \in [0,1]$, let $S = \sum_{i=1}^{n} X_i$ and $\mu = \mathbb{E}[S]$. Then:

For $\delta > 0$:
\begin{equation}
\mathbb{P}(S \geqslant (1+\delta)\mu) \leqslant \exp\left(-\frac{\delta^2 \mu}{2+\delta}\right).
\end{equation}

For $0 < \delta < 1$:
\begin{equation}
\mathbb{P}(S \leqslant (1-\delta)\mu) \leqslant \exp\left(-\frac{\delta^2 \mu}{2}\right).
\end{equation}
\end{theorem}
\begin{proof}
For the upper bound, using the moment generating function method with $t > 0$:
\begin{equation}
\mathbb{P}(S \geqslant (1+\delta)\mu) = \mathbb{P}(S - \mu \geqslant \delta\mu) = \mathbb{P}(e^{t(S-\mu)} \geqslant e^{t\delta\mu}).
\end{equation}

By Markov's inequality:
\begin{equation}
\mathbb{P}(e^{t(S-\mu)} \geqslant e^{t\delta\mu}) \leqslant \frac{\mathbb{E}[e^{t(S-\mu)}]}{e^{t\delta\mu}} = e^{-t\delta\mu} \prod_{i=1}^{n} \mathbb{E}[e^{t(X_i-\mathbb{E}[X_i])}].
\end{equation}

Since $X_i \in [0,1]$, using Bernstein's technique, for $t \geqslant 0$:
\begin{equation}
\mathbb{E}[e^{t(X_i-\mathbb{E}[X_i])}] \leqslant e^{t^2/8}.
\end{equation}

Therefore:
\begin{equation}
\mathbb{P}(S \geqslant (1+\delta)\mu) \leqslant e^{-t\delta\mu + nt^2/8}.
\end{equation}

Choosing $t = \frac{4\delta}{2+\delta}$ yields the desired bound.
\end{proof}
\begin{theorem}[Hoeffding's Inequality]
Let $X_1, X_2, \ldots, X_n$ be independent random variables such that $X_i$ almost surely lies in the interval $[a_i, b_i]$. Define $S_n = \sum_{i=1}^{n} X_i$. Then for any $t > 0$:
\begin{equation}
\mathbb{P}(S_n - \mathbb{E}[S_n] \geqslant t) \leqslant \exp\left(-\frac{2t^2}{\sum_{i=1}^{n}(b_i - a_i)^2}\right).
\end{equation}
\end{theorem}
\begin{proof}
Using the moment generating function method. For $s > 0$:
\begin{equation}
\mathbb{P}(S_n - \mathbb{E}[S_n] \geqslant t) = \mathbb{P}(e^{s(S_n - \mathbb{E}[S_n])} \geqslant e^{st}) \leqslant e^{-st} \mathbb{E}[e^{s(S_n - \mathbb{E}[S_n])}].
\end{equation}

Since $X_i$ are independent:
\begin{equation}
\mathbb{E}[e^{s(S_n - \mathbb{E}[S_n])}] = \prod_{i=1}^{n} \mathbb{E}[e^{s(X_i - \mathbb{E}[X_i])}].
\end{equation}

For bounded random variables $X_i \in [a_i, b_i]$, using Hoeffding's lemma:
\begin{equation}
\mathbb{E}[e^{s(X_i - \mathbb{E}[X_i])}] \leqslant e^{s^2(b_i - a_i)^2/8}.
\end{equation}

Therefore:
\begin{equation}
\mathbb{P}(S_n - \mathbb{E}[S_n] \geqslant t) \leqslant e^{-st} \prod_{i=1}^{n} e^{s^2(b_i - a_i)^2/8} = \exp\left(-st + \frac{s^2}{8}\sum_{i=1}^{n}(b_i - a_i)^2\right).
\end{equation}

Choosing $s = \frac{4t}{\sum_{i=1}^{n}(b_i - a_i)^2}$ yields the optimal bound. 
\end{proof}
\subsection{Arithmetic Lemmas}
\begin{lemma}[Upper Bound Lemma for t-th Power Sum on Simplex]
For any $x \in \triangle_{[K]}$, we have:

\begin{equation}
\sum_{i=1}^{K} x_i^t \leqslant K^{1-t}
\end{equation}

for $\frac{1}{2} \leqslant t < 1$.
\end{lemma}
\begin{proof}
By Hölder's inequality $\|fg\|_1 \leqslant \|f\|_p \|g\|_q$, we have:
\begin{equation}
\sum_{i=1}^{K} x_i^t \leqslant \left(\sum_{i=1}^{K} (x_i^t)^{1/t}\right)^t \left(\sum_{i=1}^{K} 1^q\right)^{1/q} = K^{1-t}
\end{equation}
by choosing $p = \frac{1}{t}$ and $q = \frac{1}{1-t}$. 
\end{proof}
\begin{lemma}
For any positive integer $n$:
\begin{equation}
\sum_{i=1}^{n} \frac{1}{i} \leqslant \ln(n + 1).
\end{equation}

Moreover, for any $-1 < t < 0$:
\begin{equation}
\sum_{i=1}^{n} i^t \leqslant \frac{(n + 1)^{t+1}}{t+1}.
\end{equation}
\end{lemma}
\begin{proof}
If $t = -1$, we have:
\begin{equation}
\sum_{i=1}^{n} i^{-1} \leqslant \int_{1}^{n+1} \frac{\mathrm{d}x}{x} = \ln(n+1).
\end{equation}

If $t > -1$, we have:
\begin{equation}
\sum_{i=1}^{n} i^t \leqslant \int_{0}^{n+1} x^t \, \mathrm{d}x = \frac{(n+1)^{t+1}}{t+1}.
\end{equation}

This completes the proof.
\end{proof}
\begin{lemma}
For any $x \geqslant 1$ and $q \in (0,1)$:
\begin{equation}
x^q - (x - 1)^q \leqslant q(x - 1)^{q - 1}.
\end{equation}
\end{lemma}
\begin{proof}Consider the function $f(x) = x^q$. We have $f''(x) = q(q - 1)x^{q - 2} \leqslant 0$ for $x \geqslant 0$ and $q \in (0,1)$. Hence, $f(x)$ is concave for $x \geqslant 0$ and $q \in (0,1)$.

By properties of concave functions:
\begin{equation}
f(x) \leqslant f(x - 1) + f'(x - 1)(x - (x - 1)) = f(x - 1) + q(x - 1)^{q - 1}
\end{equation}
for any $x \geqslant 1$ and $q \in (0,1)$, which gives $x^q - (x - 1)^q \leqslant q(x - 1)^{q - 1}$. 
\end{proof}

\section{Example of Generated Results}

Here's an example of generated results for $\alpha$=0.2 in GPT-2, 1B model. Under extreme stabilization ($\alpha=0.2$), the 1B-parameter GPT-2 reaches 100\% high-frequency word proportion, demonstrating complete stability-induced mode collapse. This aligns with our theory: over-stabilization drives forward KL minimization, concentrating probability mass on narrow empirical modes and causing linguistic structural degradation.

\begin{figure*}[htbp]
\centering
\includegraphics[width=0.9\linewidth]{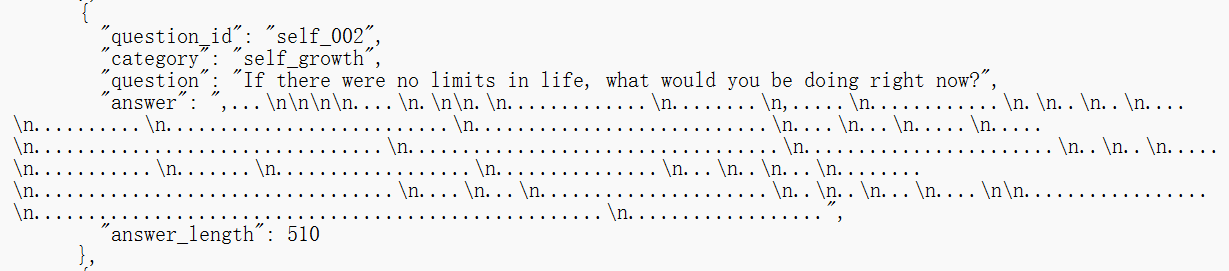}
\label{fig:gen}
\end{figure*}

\end{document}